%% file: arxiv_2026.tex
\ifdefined\XeTeXversion
\else
  \ifdefined\pdfoutput
    \pdfoutput=1
  \fi
\fi
\documentclass{article}

\PassOptionsToPackage{numbers,sort}{natbib}

\usepackage[preprint]{neurips_2026}

\usepackage[utf8]{inputenc} % allow utf-8 input
\usepackage[T1]{fontenc}    % use 8-bit T1 fonts
\usepackage{hyperref}       % hyperlinks
\usepackage{url}            % simple URL typesetting
\usepackage{booktabs}       % professional-quality tables
\usepackage{amsfonts}       % blackboard math symbols
\usepackage{amsmath}        % align, equation, etc.
\usepackage{amssymb}        % extended math symbols
\newcommand{\mathbbm}[1]{\mathbb{#1}} % arXiv-safe indicator fallback
\usepackage{graphicx}       % \includegraphics
\usepackage{wrapfig}        % text-wrapped figures
\usepackage{nicefrac}       % compact symbols for 1/2, etc.
\usepackage{microtype}      % microtypography
\usepackage{xcolor}         % colors
\usepackage{multirow}       % multi-row table cells
\usepackage{enumitem}       % customizable lists
\usepackage{amsthm}         % theorem environments

\usepackage{microtype}
\usepackage{hyperref}
\usepackage{url}
\usepackage{booktabs}
\usepackage{tcolorbox, xcolor, listings}
\tcbuselibrary{skins, breakable}
\usepackage{amsmath} 
\usepackage{lineno}
\usepackage{multirow}
\usepackage[T1]{fontenc}        % <-- fixes < > rendering in ttfamily
\usepackage{tcolorbox, xcolor}
\usepackage{caption} 
\usepackage{graphicx} 
\usepackage{colortbl}
\usepackage{amsthm}

\newtheorem{theorem}{Theorem}

\newtheorem{proposition}[theorem]{Proposition}

\providecommand{\coloneqq}{\mathrel{:=}}

\newcommand{\agentdebug}[4]{}

\definecolor{gmcolor}{RGB}{88,86,214}

\usepackage{titletoc}
\titlecontents{section}[1.5em]
    {\addvspace{6pt}\bfseries}
    {\contentslabel{1.5em}}
    {\hspace*{-1.5em}}
    {\titlerule*[0.5pc]{.}\contentspage}
\titlecontents{subsection}[3.8em]
    {}
    {\contentslabel{2.3em}}
    {\hspace*{-2.3em}}
    {\titlerule*[0.5pc]{.}\contentspage}
\hypersetup{linktoc=page}

\usepackage[capitalize,noabbrev]{cleveref}
\title{SAPO: Step-Aligned Policy Optimization for Reasoning-Based Generative Recommendation}

\author{%
  Zaiyi Zheng\textsuperscript{1},
  Liang Wu\textsuperscript{2},
  Guanghui Min\textsuperscript{1},
  Yaochen Zhu\textsuperscript{1},
  \\
  \textbf{Liangjie Hong\textsuperscript{2}},
  \textbf{Chen Chen\textsuperscript{1}},
  \textbf{Jundong Li\textsuperscript{1}}\\
  \textsuperscript{1}University of Virginia, \textsuperscript{2}Nokia\\
  \texttt{\{sjc4fq,jjm8vr,uqp4qh,zrh6du,jl6qk\}@virginia.edu}, \\
  \texttt{\{liang.wu,liangjie.hong\}@nokia.com}\\
}

\begin{document}

% #region agent log
\agentdebug{H3}{neurips_2026.tex:79}{document-started}{{"root":"neurips_2026.tex"}}
% #endregion

\maketitle

\input{sections/00_abstract_arxiv}

\input{sections/01_introduction}

\input{sections/02_preliminary}

\input{sections/025_theoretical_analysis}

\input{sections/03_methodology}

% #region agent log
\agentdebug{H3}{neurips_2026.tex:95}{methodology-loaded}{{"section":"03_methodology"}}
% #endregion

\input{sections/04_experiments}

\input{sections/06_conclusion}

% =====================================================================
\begin{ack}
% Omit for anonymous submission; fill in for camera-ready only.
\end{ack}

\bibliographystyle{plainnat}
% #region agent log
\agentdebug{H1}{neurips_2026.tex:124}{bibliography-style-selected}{{"style":"plainnat","bibfile":"main"}}
% #endregion
\bibliography{main}
% #region agent log
\agentdebug{H4}{neurips_2026.tex:118}{bibliography-finished}{{"bibfile":"main"}}
% #endregion

%%%%%%%%%%%%%%%%%%%%%%%%%%%%%%%%%%%%%%%%%%%%%%%%%%%%%%%%%%%%

\newpage
\appendix
\section*{\Large{Appendix}}
\newcommand{\appendixtoc}{
    \begingroup
    \let\oldcontentsline\contentsline
    \renewcommand{\contentsname}{\large\textbf{Appendix}}
    \startcontents[appendices]
    \printcontents[appendices]{}{1}{\setcounter{tocdepth}{2}}
    \endgroup
}
\appendixtoc
\newpage
\section{Notation and Symbols}
\label{app:notation}

We summarize the mathematical notations used throughout the paper in Table~\ref{tab:notation}.
% s
\begin{table}[h]
\centering
\caption{Summary of frequently used notation.}
\label{tab:notation}
\small
\setlength{\tabcolsep}{5pt}
\renewcommand{\arraystretch}{1.08}
\begin{tabular}{@{}llp{0.68\linewidth}@{}}
\toprule
\textbf{Symbol} & \textbf{Domain} & \textbf{Meaning} \\
\midrule
$\mathcal{C}$ & catalog & Item catalog. \\
$v, v^\star, \hat v$ & item & An item, the ground-truth next item, and the predicted item. \\
$x=(v_1,\ldots,v_{T_x})$ & sequence & Chronologically ordered user interaction history used as the sequential-recommendation prompt. \\
$K$ & integer & Number of SID codebook levels; we use $K{=}3$ in all experiments. \\
$\mathcal{V}^{(k)}$ & set & Vocabulary of the $k$-th SID codebook. \\
$\phi(v)$ & mapping & Hierarchical SID tuple assigned to item $v$. \\
$s^{(k)}$ & SID token & SID token at codebook level $k$. \\
$y_i$ & rollout & The $i$-th sampled response for a prompt in a GRPO group. \\
$\tau^{(k)}$ & token span & Reasoning span paired with the level-$k$ SID token. \\
$y_i^{(k)}$ & step unit & SAPO step consisting of $\tau_i^{(k)}$ and $s_i^{(k)}$. \\
$G$ & integer & Number of rollouts sampled per prompt for group-relative normalization. \\
$r(x,y_i)$ & scalar & Outcome reward used by standard GRPO. \\
$r_{i,k}$ & scalar & SAPO reward assigned to rollout $i$ at SID level $k$. \\
$\hat A_i, \hat A_{i,k}$ & scalar & Rollout-level GRPO advantage and step-level SAPO advantage. \\
$w_{i,t}(\theta)$ & ratio & importance ratio for token $t$ in rollout $i$. \\
$\epsilon$ & scalar & clipping radius. \\
$\alpha, \beta$ & scalar & Coefficients for the per-step match reward and final-step format bonus. \\
\bottomrule
\end{tabular}
\end{table}

\input{sections/05_related_work}

\input{sections/07_appendix}

\end{document}

%% file: sections/00_abstract_arxiv.tex
\begin{abstract}
Generative recommendation treats next-item prediction as autoregressive item-identifier generation. Specifically, items are encoded as semantic identifiers (SIDs), which are short coarse-to-fine token sequences whose early tokens capture broad semantics and later tokens refine them. Recent work augments this paradigm with reasoning traces and optimizes them via reinforcement learning with verifiable rewards, typically outcome-reward algorithm with exact-match feedback on the generated SID.
However, in large-catalog recommendation, exact-match feedback on the generated SID only reports whether the final item is correct; when a generated SID mismatches, outcome-reward cannot identify which SID-token prediction caused the mismatch and may penalize matched SID-token positions together with the mismatched position.
We identify that the natural unit of credit assignment in this setting is a single reasoning step (one thinking block paired with one SID token). We instantiate this idea in SAPO (Step-Aligned Policy Optimization): rather than broadcasting one advantage to the whole response, SAPO computes a separate group-relative advantage for each reasoning step and applies it only to the corresponding thinking block and SID token.
Across three real-world recommendation datasets, SAPO stabilizes reinforcement-learning training and consistently improves over existing generative recommendation baselines, with the largest gains where sparse exact-match feedback makes reasoning-step credit assignment important.
Our results suggest that reinforcement-learning objectives for structured generation should mirror the decoder's own decomposition of the output.
Our code is available at \url{https://github.com/zhengzaiyi/SAPO}.
% TODO(arXiv): Add the public code repository URL here once it is ready.
\end{abstract}

%% file: sections/01_introduction.tex
% =====================================================================
\section{Introduction}
\label{sec:intro}

% \GM{I see the citation format for NeurIPS is mostly like \url{https://arxiv.org/pdf/2505.10465}, maybe we can adjust accordingly.}
The item recommendation task requires a model to identify items a user is likely to interact with from a large catalog of closely related candidates.
Traditional systems typically decouple candidate retrieval from ranking, whereas generative recommendation offers an alternative by casting recommendation as autoregressive item generation.
Building on advances in language modeling~\citep{minaee2025largelanguagemodelssurvey}, this formulation has become increasingly attractive, with prior work demonstrating strong empirical performance~\citep{rajput2023recommender, deng2025onerec}.
Because assigning a unique token to every item would require a prohibitively large output vocabulary, items in this paradigm are represented as semantic identifiers (SIDs), which are short token sequences produced by a hierarchical vector quantizer. The language model is then trained to autoregressively generate the SID of the item a user will interact with next.

A complementary line of work augments generative recommenders with explicit reasoning traces, typically by prompting the model to produce a reasoning trace before SID decoding~\citep{lin2025rec, you2025r, liu2025onerec, he2026reasoning}.
These traces are intended to help the model interpret user histories and surface relevant preference signals prior to generating an item identifier.
Recent systems increasingly adopt a post-training recipe that first aligns the backbone to the recommendation vocabulary via supervised fine-tuning (SFT), then activates structured reasoning, and finally applies reinforcement learning (RL) from verifiable rewards (RLVR)~\citep{lambert2024tulu} when reliable outcome checks are available. This recipe bridges SFT-only reasoning recommenders and newer SFT-then-RL variants, including GRPO-style approaches for recommendation~\citep{zhu2025rank}. Once a response contains both a reasoning trace and a multi-token SID, the central training question becomes one of credit assignment under the match feedback of SID: \textit{which SID-token prediction caused the final mismatch?}

For hierarchical SID decoding, credit assignment should respect the positional structure of the SID hierarchy. 
While final-item verification checks only the whole SID tuple, each SID-token position reveals whether the model has captured fine-grained item information or has failed to identify the broad intent.
Outcome-reward RL discards this structure by broadcasting a scalar reward across the reasoning tokens and all SID-token predictions. 
This mismatch produces two failure modes: matched SID-token positions can be penalized together with the mismatched position, and scalar advantage broadcast cannot localize the mismatched SID-token prediction among unrelated reasoning steps. 
The underlying issue is action granularity: the rollout is too coarse, yet individual tokens are too fine-grained to serve as the credit-assignment unit for structured SID decoding.

To address this mismatch, we propose \textbf{SAPO}, a reinforcement learning method for the RL stage of reasoning-based generative recommendation. SAPO exploits the structure already present in SID decoding: each generated SID tuple exposes multiple verifiable token predictions, giving every rollout a natural decomposition into reasoning steps. We therefore define a reasoning step as one thinking block paired with its corresponding SID token and treat this pair as the action unit for reinforcement learning. Once the action unit is fixed, the rest of the update follows naturally: SAPO assigns a verifiable match reward to each reasoning step from SID-token correctness, estimates a group-relative advantage at the reasoning-step level, and optimizes a step-aligned surrogate with step-normalized token aggregation so that each reasoning step receives its own update signal. The resulting credit assignment remains consistent with the original exact-match objective while providing denser supervision when exact SID-tuple matches are sparse. Our main contributions are as follows:
\begin{itemize}[leftmargin=*,itemsep=2pt,topsep=2pt]
    \item We identify the mismatch between exact-match outcome feedback and reasoning-based SID decoding as a problem of action granularity: rollout-level rewards cannot identify the specific mismatch steps, while rollout-level advantages can penalize matched step positions (Section~\ref{sec:action-granularity}).
    \item We propose \textbf{SAPO}, which treats each reasoning step, pairing one thinking block with one SID token, as the RL action unit and derives from this choice a verifiable per-step match reward, a step-level group-relative advantage, and a step-aligned surrogate with step-normalized token aggregation, and introduces no learned reward model (Section~\ref{sec:method}).
    \item Experiments on three real-world recommendation datasets show that SAPO produces more stable RL dynamics and yields strong overall Recall and NDCG results, especially in settings where fine-grained credit assignment matters (Section~\ref{sec:experiments}).
\end{itemize}

%% file: sections/02_preliminary.tex
% =====================================================================
\section{Preliminaries}
\label{sec:prelim}
This section introduces the notation and setup used throughout the paper; Appendix~\ref{app:notation} provides a consolidated summary of symbols and notations used in the paper.

\subsection{Generative Recommendation with Hierarchical Semantic Identifiers}
\label{sec:formulation}

We formulate generative recommendation over hierarchical semantic identifiers (SIDs).
Let $\mathcal{C}$ denote the item catalog.
Following recent work on generative recommendation with SIDs~\citep{rajput2023recommender, deng2025onerec},
each item $v\in\mathcal{C}$, where $v$ denotes an item, is mapped to a tuple of $K$ SIDs.
The tuple can be obtained with a hierarchical vector-quantization scheme, such as a residual-quantization variational autoencoder (RQ-VAE) adapted from residual-quantized image generation~\citep{lee2022autoregressive}.
Let $\mathcal{V}^{(k)}$ denote the vocabulary of the $k$-th codebook for $k=1,\ldots,K$.
The quantizer defines an injection $\phi\colon\mathcal{C}\to\mathcal{V}^{(1)}\times\mathcal{V}^{(2)}\times\cdots\times\mathcal{V}^{(K)}$.
Let $\Phi\coloneqq\phi(\mathcal{C})$ denote the set of valid catalog SID tuples.
Each item $v$ is then uniquely represented as $\phi(v) = (s^{(1)}, s^{(2)}, \ldots, s^{(K)})$, where $s^{(k)}\in\mathcal{V}^{(k)}$ captures progressively finer-grained item semantics, with larger $k$ corresponding to finer semantic levels.
We set $K{=}3$, which suffices to uniquely identify every item in our catalogs.

Given a chronologically ordered user interaction history $x = (v_1, \ldots, v_{T_x})$, we study the sequential recommendation task of predicting the next item $v^\star=v_{T_x+1}$.
The history is encoded as a textual prompt that contains past interactions and a task instruction, and we learn a policy $\pi_\theta(y|x)$ parameterized by a language model with parameters $\theta$.
Its response $y$ terminates in a valid SID tuple $\phi(\hat{v})\in\Phi$, where $\hat v$ denotes the predicted next item.
In the simplest non-reasoning setting, the policy directly decodes the $K$ SID tokens.
We defer the reasoning-augmented response layout and the SAPO update to Section~\ref{sec:method}.

\subsection{Three-Stage Training Setup}
\label{sec:post-training-setup}
Motivated by recent reasoning-based generative recommenders such as OneRec-Think~\citep{liu2025onerec} and SIDReasoner~\citep{he2026reasoning}, we adopt a three-stage setup for reasoning-based SID decoding.
\textbf{Stage~1} aligns the language model to the SID vocabulary and the recommendation prompt format, so the model can produce valid SID tokens for the next-item prediction task.
\textbf{Stage~2} trains the policy to generate $K$ thinking blocks followed by the $K$ SID tokens of the predicted item, giving the response one reasoning block for each SID-token position.
\textbf{Stage~3} applies reinforcement learning (RL) from final exact-SID feedback, while the SID hierarchy exposes intermediate SID-token match signals for credit assignment.
The next section analyzes why the standard outcome-reward RL objective is misaligned with this per-step decoding structure.

%% file: sections/025_theoretical_analysis.tex
\begin{figure}[t]
    \centering
    \includegraphics[width=0.95\linewidth]{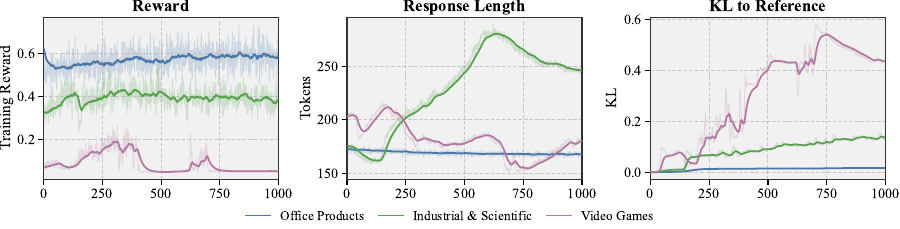}
    \caption{
        Motivating failure mode of outcome-reward GRPO across three recommendation datasets.
        Here, exact-match feedback gives a reward of 3. 
        GRPO typically improves reward early but later plateaus or becomes unstable as response length change and KL to the reference increase.
        The reward sparsity makes the instability of rollout-level credit assignment.
    }
    \label{fig:grpo-motivation}
    \vspace{-3mm}
\end{figure}

\section{The Action-Granularity Mismatch in Outcome-Reward GRPO}
\label{sec:action-granularity}

\noindent\textbf{Outcome-reward GRPO.}
We focus on the RL stage of Section~\ref{sec:post-training-setup}, where the policy already generates reasoning traces and SID tokens. 
For each prompt $x\in\mathcal{D}_{\mathrm{RL}}$, where $\mathcal{D}_{\mathrm{RL}}$ denotes the set of reinforcement-learning training prompts, outcome-reward GRPO samples $G$ rollouts $\{y_i\}_{i=1}^G$ (with $[G]=\{1,\ldots,G\}$) and assigns each rollout a scalar reward $r(x,y_i)$, typically an exact-match indicator for the ground-truth SID tuple. 
The group-relative advantage is
\begin{equation}
\label{eq:grpo-adv}
\hat{A}_{i}
=
\frac{r(x,y_i)-\operatorname{mean}_{j\in[G]} r(x,y_j)}
{\operatorname{std}_{j\in[G]} r(x,y_j)}.
\end{equation}
It then optimizes the token-level clipped policy-gradient surrogate:
\begin{equation}
\label{eq:grpo}
\mathcal{J}_{\mathrm{GRPO}}(\theta)
=
\mathbb{E}_{x,\{y_i\}}\!\left[
\frac{1}{\sum_{\ell=1}^{G}|y_\ell|}
\sum_{i=1}^{G}\sum_{t\in y_i}
\min\!\left(
w_{i,t}(\theta)\hat{A}_{i},\;
\operatorname{clip}(w_{i,t}(\theta),1{-}\epsilon,1{+}\epsilon)\hat{A}_{i}
\right)
\right],
\end{equation}
where
$w_{i,t}(\theta)=
\pi_\theta(y_{i,t}\mid x,y_{i,<t})/
\pi_{\theta_{\mathrm{old}}}(y_{i,t}\mid x,y_{i,<t})$
is the per-token importance ratio, $\theta_{\mathrm{old}}$ denotes the policy parameters before the update, and $\epsilon$ is the clipping radius.
A Kullback--Leibler (KL) penalty to the Stage~2 reference is used in practice but omitted here because it is shared by all Stage~3 variants.

\noindent\textbf{Where granularity fails.}
Eq.~\eqref{eq:grpo} assigns one scalar advantage $\hat A_i$ to the entire response.
For large-catalog SID decoding, this scalar is too coarse: final feedback reports the mismatch, but not which SID-token prediction caused it.
Consider $K=3$ with ground-truth tuple
$(s_\star^{(1)},s_\star^{(2)},s_\star^{(3)})$.
The near-miss prediction with one last mismatch token
$(s_\star^{(1)},s_\star^{(2)},\tilde s^{(3)})$
and the coarse-error prediction
$(\tilde s^{(1)},\tilde s^{(2)},\tilde s^{(3)})$
both receive zero exact-match reward.
Outcome-reward GRPO therefore assigns them the same rollout-level advantage and broadcasts it to every token; in the near miss, matched SID-token predictions are updated together with the mismatched prediction.

\noindent\textbf{Why reward shaping is insufficient.}
Figure~\ref{fig:grpo-motivation} illustrates this failure mode across all three recommendation datasets: reward dynamics plateau or become irregular as response length, KL to the reference, and format errors grow.
Reward sparsity amplifies this problem because exact SID-tuple matches are rare, making group-relative advantages weak or noisy, especially on harder recommendation datasets.
However, sparsity is not the whole issue.
A denser rollout-level reward could replace the exact-match indicator; for example, one could use
\begin{equation}
    r_{\mathrm{shape}}(x,y_i)
=
\sum_{k=1}^K
\mathbbm{1}\!\left[s_i^{(k)}=s_\star^{(k)}\right],
\end{equation}
but Eq.~\eqref{eq:grpo} would still convert this score into a single rollout-level advantage and assign it to every token.
Such shaping tells that a rollout is partially correct, but it still cannot isolate the mismatched SID-token prediction.
\textbf{This is what we refer to as the action-granularity mismatch.}
The rollout is too coarse because it hides which of the $K$ verifiable SID-token predictions succeeded or failed.
Individual response tokens are too fine because arbitrary reasoning tokens do not have verifiable correctness labels, whereas SID-token positions do.
The natural credit-assignment unit lies in between: the reasoning step that pairs a thinking block with its SID token.

This mismatch causes two concrete problems.
First, \emph{matched SID-token credit is discarded}: matched SID-token predictions receive no positive signal when another SID token is wrong, so plausible near misses are treated as complete failures.
Second, \emph{reasoning-step advantage localization is lost}: an error at one SID-token position assigns the same scalar weight to token gradients from unrelated thinking blocks and SID-token positions.
In long reasoning responses, this diffuse broadcast can reinforce or penalize many reasoning tokens together even though the verifier only checks the SID-token predictions, contributing to length drift, format degradation, and weak group-relative contrast.
These observations motivate aligning reward, advantage, and token aggregation at the reasoning-step granularity in the next section.

%% file: sections/03_methodology.tex
% =====================================================================
\section{Methodology}
\label{sec:method}

Motivated by the analysis above, SAPO addresses the RL-stage credit-assignment problem in reasoning-based SID decoding.
Its central design choice is to use each reasoning step, defined as a SID-token prediction together with its associated thinking block, as the credit-assignment (action) unit.
Once this reasoning step is fixed, reward placement, advantage estimation, and token aggregation should all be defined at the same granularity.
We now define the reasoning-step decomposition used by SAPO and present the reasoning-step match reward and the step-aligned surrogate. The overall pipeline is illustrated in Figure~\ref{fig:sapo-overview}.

\subsection{SAPO Overview and Reasoning-Step Decomposition}
\label{sec:sapo-overview}

Reasoning-augmented generative recommenders ask the policy to generate structured reasoning before producing SID tokens.
In our implementation, the model first generates all $K$ thinking blocks and then the $K$ SID tokens as a contiguous block:
\begin{equation}
\label{eq:response-impl}
y \;=\; \bigl[\,\underbrace{\langle\texttt{think}\rangle\,\tau^{(1)}\,\langle/\texttt{think}\rangle\,\oplus\,\cdots\,\oplus\,\langle\texttt{think}\rangle\,\tau^{(K)}\,\langle/\texttt{think}\rangle}_{\text{reasoning trace with $K$ thinking blocks}}\,\oplus\,\underbrace{s^{(1)}\,s^{(2)}\,\cdots\,s^{(K)}}_{\text{all $K$ SID tokens}}\,\bigr].
\end{equation}
This layout enables prefix-tree (trie)-based constrained decoding with key-value (KV) cache reuse across SID-token positions.
For credit assignment, SAPO pairs the $k$-th thinking block with the $k$-th SID token and treats that pair as \emph{reasoning step~$k$}.
Formally, under Eq.~\eqref{eq:response-impl} we define reasoning step~$k$ as $(\tau^{(k)}, s^{(k)})$, write $|\tau^{(k)}|$ for the number of tokens in the $k$-th thinking block, and define $|y^{(k)}|=|\tau^{(k)}|+1$ as the size of the paired unit including its SID token.
This pairing is used only to assign credit under the blocked decoding layout; it does not change the autoregressive factorization or causal decoding order.
Appendix~\ref{app:decoding-order} details the implementation layout and explains why SAPO's reward and advantage remain well defined under this pairing.

\subsection{Reasoning-Step Match Reward}
\label{sec:reward-design}

With the reasoning step fixed, reward should be placed at SID-token positions rather than only at the end of the rollout.
For a rollout $y_i$ with generated SID sequence $(s^{(1)}_i, \ldots, s^{(K)}_i)$ and ground truth $(s^{\text{gt},(1)}, \ldots, s^{\text{gt},(K)})$, SAPO uses a reasoning-step match reward term plus a small final format bonus:
\begin{equation}
\label{eq:step-reward}
r_{i,k} \;=\; \alpha\cdot\mathbbm{1}\!\left[s^{(k)}_i = s^{\text{gt},(k)}\right]\;+\;\beta\cdot\mathbbm{1}\!\left[k=K\right]\cdot b_i,
\end{equation}
where $\mathbbm{1}[\cdot]$ is the indicator function, $b_i=1$ when $y_i$ contains exactly $K$ extracted thinking blocks and ends in a valid SID tuple, and $b_i=0$ otherwise.
The first term preserves credit for matched SID-token positions when other SID token mismatches, while the second term attaches a small bonus to the last SID-token position when the rollout is structurally well formed.
% We set $\alpha{=}1.0$ and $\beta{=}0.2$ throughout.
The following objective-consistency discussion focuses on the reasoning-step match reward term.

\noindent\textbf{Objective consistency.}
Let $m_{i,k}\coloneqq\alpha\,\mathbbm{1}[s_i^{(k)}=s_i^{\mathrm{gt},(k)}]$ denote the match component of Eq.~\eqref{eq:step-reward}.
Although SAPO decomposes exact-match feedback into reasoning-step rewards, it does not change the exact-match optimum under the standard realizability condition:
\begin{equation}
\label{eq:lcp-identity}
\sum_{k=1}^{K} m_{i,k}
=
\alpha\,\bigl|\{k\in[K]:s_i^{(k)}=s_i^{\mathrm{gt},(k)}\}\bigr|.
\end{equation}
Let $\mathcal{J}_{\mathrm{out}}$ be the expected exact-match reward and $\mathcal{J}_{\mathrm{match}}$ be the expected cumulative match reward.
\begin{proposition}[Reasoning-step match reward and outcome consistency]
\label{prop:objective-consistency}
If the policy class $\Pi$ contains an exact-match policy, then
\begin{equation}
\label{eq:argmax-consistency}
\arg\max_{\pi\in\Pi}\mathcal{J}_{\mathrm{out}}(\pi)
=
\arg\max_{\pi\in\Pi}\mathcal{J}_{\mathrm{match}}(\pi)
\end{equation}
\end{proposition}
Thus the reasoning-step match reward supplies denser supervision while preserving the exact-match optimum; Appendix~\ref{app:proofs-lcp} gives the proof and further discussion.

\begin{figure*}[t]
\centering
\includegraphics[width=\textwidth]{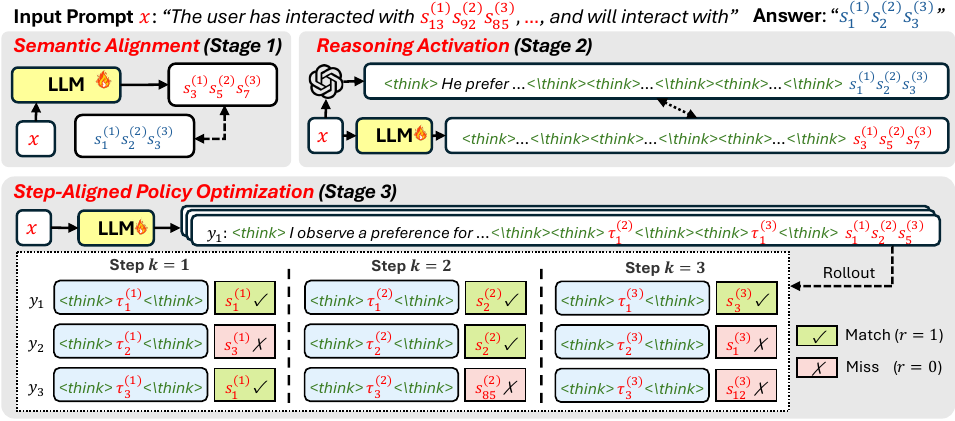}
\caption{Overview of \textbf{SAPO}.
Stage~1 aligns the language model to the SID vocabulary, Stage~2 activates level-aware reasoning, and Stage~3 applies step-aligned reinforcement learning.
Under the released decoding layout, the policy generates $K$ thinking blocks followed by the $K$ SID tokens of the predicted item.
SAPO assigns a reasoning-step match reward using SID-token correctness, attaches a small format bonus at the last SID-token position, computes a reasoning-step group-relative advantage, and optimizes a surrogate with step-normalized token aggregation.}
\label{fig:sapo-overview}
\vspace{-3mm}
\end{figure*}

\subsection{Reasoning-Step Advantage and SAPO Surrogate}
\label{sec:step-aligned-grpo}

Once reward is placed at SID-token positions, both the advantage and the surrogate loss should be defined at the same granularity.
SAPO therefore defines advantages over reasoning steps $(\tau^{(k)}, s^{(k)})$, rather than over the whole sequence (too coarse) or each individual token (too fine-grained).
Given the return $r_{i,k}$ at SID-token position $k$ in Eq.~\eqref{eq:step-reward}, we compute a group-relative advantage and its step-normalized form independently for each reasoning step:
\begin{equation}
\label{eq:step-adv}
\hat{A}_{i,k}
=
\frac{r_{i,k} - \operatorname{mean}_{i'\in[G]}(r_{i',k})}
{\operatorname{std}_{i'\in[G]}(r_{i',k})},
\quad
\tilde{A}_{i,k}
=
\frac{\hat{A}_{i,k}}{|y_i^{(k)}|}.
\end{equation}
Here ``step-normalized'' means that the reasoning-step advantage is divided by the paired unit length $|y_i^{(k)}|$, which includes the $k$-th thinking block and its paired SID token.
Using the same per-token importance ratio $w_{i,t}(\theta)$ as in Eq.~\eqref{eq:grpo}, SAPO optimizes the resulting step-aligned surrogate with step-normalized token aggregation,
% \zy{Following Rank-GRPO, use colors and braces to denote different components of the formula, with particular emphasis on highlighting the distinctions from GRPO.}
\begin{equation}
\label{eq:step-aligned-obj}
\mathcal{J}_{\mathrm{SAPO}}(\theta)
= \mathbb{E}_{x,\{y_i\}}\!\left[
\frac{1}{\sum_{\ell=1}^{G}|y_\ell|}
\sum_{i=1}^{G}\sum_{k=1}^{K}\sum_{t\in y_i^{(k)}} \right.
\left.
\min\!\left(
w_{i,t}(\theta)\tilde{A}_{i,k},\;
\operatorname{clip}(w_{i,t}(\theta),1{-}\epsilon,1{+}\epsilon)\tilde{A}_{i,k}
\right)
\right].
\end{equation}

\noindent\textbf{Gradient analysis.}
Ignoring clipping for exposition, Eq.~\eqref{eq:step-aligned-obj} gives the gradient view:
\begin{equation}
\label{eq:step-grad-view}
\begin{aligned}
\nabla_\theta \mathcal{J}_{\mathrm{SAPO}}(\theta)
&\propto
\mathbb{E}\!\left[
\frac{1}{\sum_{\ell=1}^{G}|y_\ell|}
\sum_{i=1}^{G}\sum_{k=1}^{K}\sum_{t\in y_i^{(k)}}
\,g_{i,k,t}(\theta)
\right],\\
g_{i,k,t}(\theta)
&=
\underbrace{\textcolor{blue}{w_{i,t}(\theta)}}_{\text{importance}}
\cdot
\underbrace{\textcolor{magenta}{\hat{A}_{i,k}}}_{\text{reasoning-step advantage}}
\cdot
\underbrace{\textcolor{cyan}{\frac{1}{|y_i^{(k)}|}\nabla_\theta\log\pi_\theta(y_{i,t}\mid x,y_{i,<t})}}_{\text{step-normalized token gradient}}.
\end{aligned}
\end{equation}
This decomposition highlights the two differences from rollout-level GRPO.
First, SAPO replaces the rollout-wide scalar advantage with the reasoning-step signal $\hat{A}_{i,k}$, so rollouts that differ only at SID-token position~$k$ produce a contrastive update at that step without assigning the same advantage to unrelated SID-token positions.
Second, the factor $1/|y_i^{(k)}|$ averages the token gradients within each reasoning step, ensuring that the update magnitude depends on the quality of the reasoning step rather than the number of tokens it contains.
Together, these two changes explain why SAPO mitigates the excessive reasoning verbosity and length drift observed under rollout-level advantage broadcast in Section~\ref{sec:rl-dynamics}.

%% file: sections/04_experiments.tex
% =====================================================================
\section{Experiments}
\label{sec:experiments}
% [Polish] Original: We aim to answer five research questions.
We address three research questions (RQs).
\textbf{(RQ1)}~Does SAPO improve end-task recommendation quality against generative recommendation baselines?
% [Polish] Original: \textbf{(RQ2)}~Does it translate into better end-task recommendation quality against outcome-reward baselines and other reasoning-based generative recommenders?
\textbf{(RQ2)}~{Does reasoning-step credit assignment stabilize reinforcement learning dynamics for recommendation with sparse exact-match feedback?}
\textbf{(RQ3)}~How do the reasoning-step match reward and reasoning-step advantage contribute to SAPO?
% Future extension:
% \textbf{(RQ4)}~How sensitive is SAPO to its key hyperparameters?
% [Polish] Original: \textbf{(RQ5)}~Does SAPO generalize across LLM backbones of different scales?
% \textbf{(RQ5)}~Does SAPO generalize across LLM backbones of different scales and families?

% \subsection{Datasets}
% \label{sec:datasets}
\subsection{Experimental Setup}
\noindent\textbf{Datasets.}
We evaluate on three categories from the Amazon Reviews dataset~\citep{ni2019justifying}, including \texttt{Office-Products}, \texttt{Video-Games}, and \texttt{Industrial-and-Scientific}, following the widely adopted sequential recommendation protocol of~\citet{kang2018self}.
% [Polish] Original: For each category, we sort user interactions chronologically, keep users with at least five interactions, and construct (history, next-item) pairs with the leave-one-out split: the last interaction is used for test, the second-to-last for validation, and the remaining interactions for training.
For each category, we sort user interactions chronologically, keep users with at least five interactions, and construct (history, next-item) pairs under a leave-one-out split: the last interaction forms the test record, the second-to-last forms the validation record, and the remaining prefix forms the training set.
% [Polish] Original: For every item we learn a $K{=}3$-level RQ-VAE codebook over the item text embeddings (title + category + brand + description) and use the resulting SID tuple as the target.
For every item, we learn a $K{=}3$-level RQ-VAE codebook over the item text embeddings (title, category, brand, and description) and use the resulting SID tuple as the target.
% [Polish] Original: Detailed statistics and the SID tokenization procedure are deferred to Appendix~\ref{app:data}; concrete record and prompt examples are provided in Appendix~\ref{app:data-examples}.
Detailed statistics and the SID tokenization procedure are deferred to Appendix~\ref{app:data}; concrete record and prompt examples are provided in Appendix~\ref{app:data-examples}.
% [Polish] Original: Implementation details, backbone variants, hyperparameters, and reproducibility information are deferred to Appendix~\ref{app:implementation}.
Implementation details, backbone, hyperparameters, and reproducibility information are deferred to Appendix~\ref{app:implementation}.
To keep the main text concise, single-category figures and diagnostics use \texttt{Industrial-and-Scientific} as the representative dataset.

% \subsection{Baselines}
% \label{sec:baselines}
\noindent\textbf{Baselines.}
% [Polish] Original: We compare SAPO against representative methods from three families.
We compare SAPO against representative methods from three families.
\textbf{(i)~Traditional sequential recommenders}: GRU4Rec~\citep{hidasi2015session}, SASRec~\citep{kang2018self}, and Caser~\citep{tang2018personalized}.
\textbf{(ii)~Generative recommenders without reasoning}: TIGER~\citep{rajput2023recommender}, HSTU~\citep{zhai2024actions}, LCRec~\citep{zheng2024adapting}, and LETTER~\citep{wang2024learnable}.
% [Polish] Original: \textbf{(iii)~Reasoning-based generative recommenders trained with outcome-reward RL}: SIDReasoner~\citep{he2026reasoning} (our most direct head-to-head baseline, which shares our Stages 1 and 2 but uses outcome-reward RL in Stage~3), OneRec-Think~\citep{liu2025onerec}, R$^2$ec~\citep{you2025r}, and ReaRec~\citep{tang2025think}.
\textbf{(iii)~Reasoning-based generative recommenders}: SIDReasoner~\citep{he2026reasoning}, OneRec-Think~\citep{liu2025onerec}, R$^2$ec~\citep{you2025r}, and ReaRec~\citep{tang2025think}.
% [Polish] Original: All methods are re-implemented and evaluated under the same protocol for fair comparison. A brief description of each baseline is provided in Appendix~\ref{app:baselines}.
We report SIDReasoner results by evaluating the official checkpoints provided by the authors, and re-implement all other baselines for fair comparison.
Brief descriptions of all baselines are provided in Appendix~\ref{app:baselines}.

\subsection{Comparison with Baselines (RQ1)}
\label{sec:main-results}

% [Polish] Original: To answer \textbf{RQ2}, we evaluate SAPO against the baselines described in Section~\ref{sec:baselines} on the held-out test set of each category.
To address \textbf{RQ1}, we evaluate SAPO against the baselines on the held-out test set of each category.
% [Polish] Original: Following common practice~\citep{kang2018self}, we report Recall@$k$ (R@$k$) and NDCG@$k$ (N@$k$) at $k\!\in\!\{5, 10, 20\}$.
Following common practice~\citep{kang2018self}, we report Recall@$k$ (R@$k$) and normalized discounted cumulative gain (NDCG@$k$, denoted N@$k$) at $k\!\in\!\{5, 10\}$.

\begin{table}[t]
\centering
\caption{Main results on three Amazon-review datasets.
R@$k$ and N@$k$ denote Recall@$k$ and NDCG@$k$, respectively.
The best result in each column is \textbf{bolded} and the second best is \underline{underlined}.
% All methods are re-implemented and evaluated under our unified protocol for fair comparison.
}
\label{tab:main}
\resizebox{\textwidth}{!}{%
\begin{tabular}{@{}l cc cc cc cc cc cc@{}}
\toprule
& \multicolumn{4}{c}{\textbf{\texttt{Office-Products}}}
& \multicolumn{4}{c}{\textbf{\texttt{Video-Games}}}
& \multicolumn{4}{c}{\textbf{\texttt{Industrial-and-Scientific}}} \\
\cmidrule(lr){2-5} \cmidrule(lr){6-9} \cmidrule(lr){10-13}
\textbf{Method}
& R@5 & N@5 & R@10 & N@10
& R@5 & N@5 & R@10 & N@10
& R@5 & N@5 & R@10 & N@10 \\
\midrule
\multicolumn{13}{@{}l}{\textit{(i) Traditional sequential recommenders}} \\
GRU4Rec& 0.0240 & 0.0135 & 0.0386 & 0.0181 & 0.0207 & 0.0139 & 0.0357 & 0.0187 & 0.0457 & 0.0408 & 0.0532 & 0.0432 \\
SASRec& 0.1019 & 0.0824 & 0.1167 & 0.0871 & 0.0501 & 0.0345 & 0.0723 & 0.0416 & 0.0807 & 0.0647 & 0.0964 & 0.0697 \\
% BERT4Rec& -- & -- & -- & -- & -- & -- & -- & -- & -- & -- & -- & -- \\
Caser& 0.0232 & 0.0139 & 0.0347 & 0.0176 & 0.0247 & 0.0161 & 0.0417 & 0.0214 & 0.0459 & 0.0377 & 0.0578 & 0.0415 \\
\midrule
\multicolumn{13}{@{}l}{\textit{(ii) Generative recommenders without reasoning}} \\
TIGER& 0.0569 & 0.0415 & 0.0729 & 0.0467 & 0.0475 & 0.0328 & 0.0692 & 0.0397 & 0.0979 & 0.0792 & 0.1291 & 0.0885 \\
HSTU& 0.1204 & 0.1069 & 0.1323 & 0.1107 & 0.0539 & 0.0396 & 0.0746 & 0.0462 & 0.1008 & 0.0898 & 0.1138 & 0.0940 \\
LCRec& 0.0193 & 0.0135 & 0.0370 & 0.0191 & 0.0158 & 0.0115 & 0.0177 & 0.0121 & 0.0552 & 0.0512 & 0.0600 & 0.0527 \\
LETTER
& 0.1315 & 0.1074 & 0.1520 & 0.1139   % Office
& 0.0445 & 0.0294 & 0.0709 & 0.0378   % Video (Games)
& 0.1080 & 0.0850 & 0.1389 & 0.0950 \\ % Industrial
% CoRAB& -- & -- & -- & -- & -- & -- & -- & -- & -- & -- & -- & -- \\
\midrule
\multicolumn{13}{@{}l}{\textit{(iii) Reasoning-based generative recommenders}} \\

R$^2$ec
& 0.1147 & 0.0894 & 0.1486 & 0.1004   % Office
& \textbf{0.0655} & \underline{0.0399} & \underline{0.0931} & \underline{0.0525}   % Video (Games)
& 0.0880 & 0.0774 & 0.1253 & 0.0774 \\ % Industrial

ReaRec
& 0.1173 & 0.0988 & 0.1385 & 0.1057
& 0.0568 & 0.0381 & 0.0843 & 0.0470
& 0.0973 & 0.0796 & 0.1205 & 0.0870 \\

OneRec-Think
& 0.1353 & 0.1107 & 0.1506 & 0.1166 & 0.0413 & 0.0296 & 0.0575 & 0.0349 & 0.0464 & 0.0398 & 0.0520 & 0.0416 \\

SIDReasoner
& \textbf{0.1381} & \underline{0.1116} & \underline{0.1629} & \underline{0.1197}
& 0.0578 & 0.0385 & 0.0897 & 0.0487
& \underline{0.1109} & \underline{0.0905} & \textbf{0.1438} & \underline{0.1010} \\
\midrule
% \multicolumn{13}{@{}l}{\textit{(iv) SAPO (ours)}} \\
\rowcolor{gray!20}
SAPO 
& \underline{0.1362} & \textbf{0.1141} & \textbf{0.1643} & \textbf{0.1245}
& \underline{0.0620} & \textbf{0.0413} & \textbf{0.0986} & \textbf{0.0529} 
& \textbf{0.1176} & \textbf{0.0931} & \underline{0.1416} & \textbf{0.1012} \\
\bottomrule
\end{tabular}%
\vspace{-5mm}
}
\end{table}

Table~\ref{tab:main} shows that SAPO improves end-task ranking quality, while the pattern is more informative than uniform dominance.
The most direct comparison is SIDReasoner, which shares the same broad reasoning-based setup but uses a rollout-level shaped reward in Stage~3: SAPO improves every NDCG column over SIDReasoner and improves R@10 on two of the three categories, while remaining slightly below SIDReasoner on Office-Products R@5 and Industrial-and-Scientific R@10.
This pattern suggests that reasoning activation alone is not sufficient; once reasoning is introduced, the RL credit signal must also respect the hierarchical SID-token structure.
Across the full table, SAPO achieves the best NDCG on all  columns, indicating that step-aligned credit assignment improves the ordering of the ground-truth item within the top-ranked list.
At the same time, SAPO ranks first or second on all reported Recall columns rather than uniformly dominating every Recall setting.
SAPO is consistently stronger than non-reasoning generative baselines, supporting the broader conclusion that reasoning benefits most when RL credit assignment follows the reasoning-step structure.

\begin{figure}[t]
    \centering
    \includegraphics[width=\linewidth]{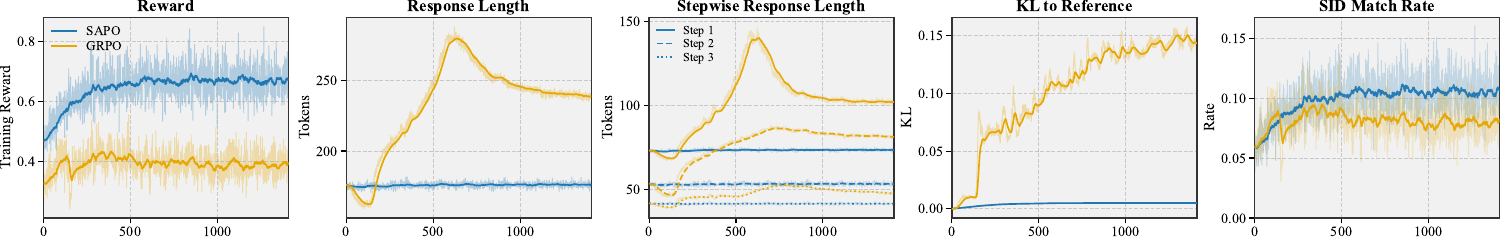}
    \caption{RL training dynamics on the Industrial-and-Scientific dataset.
    We compare SAPO and GRPO from the same Stage~2 checkpoint using training reward, response length, stepwise response length, KL to the reference, and SID match rate; blue denotes SAPO and orange denotes GRPO.}
    \label{fig:diagnostic}
    \vspace{-5mm}
\end{figure}

\subsection{Training Dynamics (RQ2)}
\label{sec:rl-dynamics}
We next address \textbf{RQ2} by comparing the training dynamics of GRPO and SAPO from the same Stage~2 checkpoint.
Figure~\ref{fig:diagnostic} reports the representative \texttt{Industrial-and-Scientific} run using five signals: training reward, response length, stepwise response length, KL to the reference, and SID match rate.
SAPO shows a steadier reward trajectory and stronger SID-match dynamics, while keeping both overall reasoning-trace length and per-reasoning-step thinking-block lengths more controlled.
The per-reasoning-step length curves further suggest that this control does not come from shortening only one thinking block, but from stabilizing the reasoning budget across SID-token positions.
GRPO exhibits noisier and less stable reward dynamics under sparse exact-match feedback, together with larger length and KL drift; Appendix~\ref{app:diagnostic-extra} provides the corresponding diagnostics for the other two datasets.
The SID match-rate panel links these optimization dynamics to the target hierarchy, showing whether training actually improves SID-token predictions rather than only increasing rollout-level reward.
Taken together, these dynamics are consistent with the mechanism in Section~\ref{sec:step-aligned-grpo}: SAPO ties updates to reasoning-step advantages and step-normalized token aggregation, rather than broadcasting one scalar advantage over the whole rollout.

\begin{table}[t]
    \centering
    \caption{Component ablation of SAPO on three Amazon-review datasets. The per-step labels correspond to the reasoning-step advantage and reasoning-step match reward defined at SID-token positions. We additionally report the reasoning activation performance on stage 2 for reference.
    The best result in each column is \textbf{bolded} and the second best is \underline{underlined}.
    % \zy{Alternatively, should we switch to a bar chart?}
    }
    \label{tab:ablation}
    \resizebox{\textwidth}{!}{%
    {\small
    \setlength{\tabcolsep}{3pt}
    \begin{tabular}{@{}l cc cc cc cc cc cc@{}}
    \toprule
    & \multicolumn{4}{c}{\textbf{\texttt{Office-Products}}}
    & \multicolumn{4}{c}{\textbf{\texttt{Video-Games}}}
    & \multicolumn{4}{c}{\textbf{\texttt{Industrial-and-Scientific}}} \\
    \cmidrule(lr){2-5} \cmidrule(lr){6-9} \cmidrule(lr){10-13}
    \textbf{Method}
    & R@5 & N@5 & R@10 & N@10
    & R@5 & N@5 & R@10 & N@10
    & R@5 & N@5 & R@10 & N@10 \\
    \midrule
    Reasoning Activation
    & 0.1143 & 0.0933 & 0.1410 & 0.1020
    & 0.0577 & 0.0375 & 0.0926 & 0.0466
    & 0.1032 & 0.0812 & 0.1299 & 0.0899 \\
    \textbf{Full SAPO}
    & \textbf{0.1362} & \textbf{0.1141} & \textbf{0.1643} & \textbf{0.1245}
    & \underline{0.0620} & \textbf{0.0413} & \textbf{0.0986} & \textbf{0.0529}
    & \textbf{0.1176} & \textbf{0.0931} & \textbf{0.1416} & \textbf{0.1012} \\
    w/o per-step advantage
    & 0.1229 & 0.1008 & 0.1494 & 0.1093
    & \textbf{0.0625} & \underline{0.0411} & \underline{0.0979} & \underline{0.0528}
    & \underline{0.1171} & \underline{0.0916} & \underline{0.1401} & \underline{0.0990} \\
    w/o per-step match reward
    & \underline{0.1268} & \underline{0.1029} & \underline{0.1535} & \underline{0.1115}
    & 0.0591 & 0.0385 & 0.0941 & 0.0498
    & 0.1149 & 0.0913 & 0.1389 & \underline{0.0990} \\
    w/o both \textit{(pure GRPO)}
    & 0.1235 & 0.1001 & 0.1490 & 0.1083
    & 0.0290 & 0.0209 & 0.0427 & 0.0256
    & 0.1032 & 0.0846 & 0.1207 & 0.0903 \\
    \bottomrule
    \end{tabular}}
    }
    \vspace{-3mm}
\end{table}

\subsection{Ablation Study (RQ3)}
\label{sec:ablation}

To address \textbf{RQ3}, we run component ablations on the three categories and report all completed entries in Table~\ref{tab:ablation}; single-category training diagnostics use \texttt{Industrial-and-Scientific} as the representative setting.
Starting from full SAPO, we retrain Stage~3 after removing one credit-assignment ingredient at a time.
Specifically, we consider the following three variants:
\textit{\textbf{(i) w/o per-step group-relative advantage}}, removes reasoning-step advantage localization by replacing $\hat{A}_{i,k}$ with a single sequence-level advantage while keeping the per-step match reward and step-normalized token assignment.
\textit{\textbf{(ii) w/o per-step match reward}}, removes dense SID-token feedback by replacing the per-step match reward with the final SID-tuple exact-match reward shared by all steps.
\textit{\textbf{(iii) w/o both (pure GRPO)}}, removes both ingredients and broadcasts one rollout-level advantage to all response tokens.

\setlength{\columnsep}{12pt}
\begin{wrapfigure}[12]{r}{0.6\textwidth}
\vspace{-0.55cm}
  \begin{center}
    \includegraphics[width=0.29\textwidth]{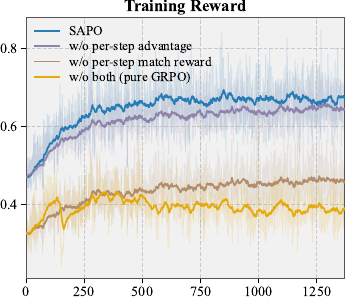}
    \hfill
    \includegraphics[width=0.294\textwidth]{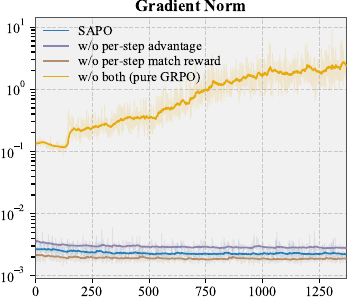}
    \vspace{-0.3cm}
  \end{center}
  \caption{Ablation diagnostics: reward and gradient norm (log$_{10}$ scale) on the Industrial \& Scientific dataset.}
  \vspace{-0.2cm}
  \label{fig:ablation-diagnostics}
\end{wrapfigure}

\noindent
Table~\ref{tab:ablation} shows that SAPO's two ingredients are complementary rather than interchangeable.
Across the completed ablations, full SAPO remains the strongest setting, while removing either per-step group-relative advantages or per-step match rewards consistently weakens ranking quality.
The drops are moderate for the single-component ablations but larger when both ingredients are removed, indicating that dense SID-token feedback and reasoning-step advantage localization address different parts of the credit-assignment problem.

\noindent
Beyond final ranking metrics, Figure~\ref{fig:ablation-diagnostics} clarifies why this complementarity matters during RL.
The single-component ablations do not exhibit the gradient-norm explosion of pure GRPO, suggesting that either retaining the reasoning-step match reward or retaining the step-aligned surrogate is enough to avoid the worst rollout-level update amplification.
However, their reward trajectories and final ranking metrics remain below full SAPO, which suggests that stability alone is not sufficient: the best performance requires both a denser SID-token match signal and reasoning-step advantage assignment.
Thus, SAPO improves recommendation quality not by adding an isolated stabilizer, but by preserving a fully step-aligned credit structure.

\begin{figure}[t]
    \centering
    \newcommand{\casesid}[1]{\textcolor{red!70!black}{\texttt{#1}}}
    \begin{tcolorbox}[
        enhanced,
        width=\linewidth,
        colback=gray!4,
        colframe=gray!55,
        boxrule=0.4pt,
        arc=1mm,
        left=3pt,
        right=3pt,
        top=3pt,
        bottom=3pt,
        fontupper=\scriptsize
    ]
    \textbf{Shared prompt history.}
    \casesid{<a\_15><b\_155><c\_49>} (Industrial \& Scientific category);
    \casesid{<a\_91><b\_29><c\_123>} (Kamo 3D-printer extruder brass nozzle);
    \casesid{<a\_73><b\_244><c\_79>} (pin-header and Dupont connector kit);
    \casesid{<a\_158><b\_206><c\_3>} (black PLA 3D-printer filament);
    \casesid{<a\_159><b\_179><c\_39>} (NEMA17 flexible coupling for CNC);
    \casesid{<a\_30><b\_192><c\_4>} (NEMA17 stepper motor for DIY CNC).\\[-1pt]
    \textbf{Target.}
    \casesid{<a\_158><b\_89><c\_233>}:
    DRV8825 StepStick stepper motor drivers for 3D-printer electronics, CNC machines, or robotics.
    \end{tcolorbox}
    \vspace{2pt}
    \begin{minipage}[t]{0.56\linewidth}
    \begin{tcolorbox}[
        enhanced,
        width=\linewidth,
        colback=red!3,
        colframe=red!45!black,
        colbacktitle=red!12,
        coltitle=black,
        title=\textbf{Outcome GRPO},
        fonttitle=\scriptsize,
        fontupper=\tiny,
        boxrule=0.4pt,
        arc=1mm,
        left=3pt,
        right=3pt,
        top=3pt,
        bottom=3pt
    ]
    \textbf{Thinking block 1.}
    I notice the user has a strong interest in 3D printing and related hardware, as seen in their engagement with \casesid{<a\_91><b\_29><c\_123>}, \casesid{<a\_158><b\_206><c\_3>}, and \casesid{<a\_159><b\_179><c\_39>}. They also explored electronics components like \casesid{<a\_73><b\_244><c\_79>} and \casesid{<a\_30><b\_192><c\_4>}, indicating a hands-on approach to building and customizing devices. The inclusion of \casesid{<a\_15><b\_155><c\_49>} suggests an appreciation for precision and quality in their projects.\\[1pt]
    \textbf{Thinking block 2.}
    Given the variety of items--ranging from filament and extruder parts to electronic connectors and stepper motors--the user appears to enjoy assembling, upgrading, and maintaining 3D printers and related equipment. Their pattern of selecting both consumables and mechanical components points to a methodical, project-oriented mindset, possibly for DIY or small-scale manufacturing. I would expect the user to continue exploring compatible accessories, upgrade parts, or additional 3D printing materials that support their technical interests.\\[1pt]
    \textbf{Thinking block 3.}
    Overall, the user may enjoy further items related to 3D printer hardware, electronics, or precision tools that enable more advanced or varied project builds. I will recommend the next item as \casesid{<a\_91><b\_105><c\_93>}, which aligns with their ongoing interest in high-quality, compatible components for 3D printing.\\[2pt]
    \textbf{Prediction.}
    \casesid{<a\_91><b\_105><c\_93>}:
    extruder nozzle and tube kit.
    \textit{\textbf{Wrong; the target is not in top-10.}}
    \end{tcolorbox}
    \end{minipage}\hfill
    \begin{minipage}[t]{0.42\linewidth}
    \begin{tcolorbox}[
        enhanced,
        width=\linewidth,
        colback=blue!3,
        colframe=blue!50!black,
        colbacktitle=blue!12,
        coltitle=black,
        title=\textbf{SAPO},
        fonttitle=\scriptsize,
        fontupper=\tiny,
        boxrule=0.4pt,
        arc=1mm,
        left=3pt,
        right=3pt,
        top=3pt,
        bottom=3pt
    ]
    \textbf{Thinking block 1.}
    Examining the user's interaction history, I see a clear pattern of engagement with components and accessories central to 3D printing and CNC projects. The user consistently selects items such as \casesid{<a\_91><b\_29><c\_123>}, \casesid{<a\_158><b\_206><c\_3>}, and \casesid{<a\_159><b\_179><c\_39>}, which are directly related to filament, extrusion, and mechanical motion--key elements in building or maintaining 3D printers.\\[1pt]
    \textbf{Thinking block 2.}
    The inclusion of \casesid{<a\_73><b\_244><c\_79>} and \casesid{<a\_30><b\_192><c\_4>} suggests an interest in electronics and stepper motor control, indicating a hands-on approach to hardware assembly and customization. The user's choices reflect a preference for reliable, compatible, and high-quality parts, often in bulk or multi-pack formats, which points to ongoing or large-scale projects.\\[1pt]
    \textbf{Thinking block 3.}
    Based on these patterns, the user may enjoy exploring further 3D printer upgrades, motion control components, or specialized accessories that enhance printer performance and expand functionality.\\[2pt]
    \textbf{Prediction.}
    \casesid{<a\_158><b\_89><c\_233>}:
    DRV8825 StepStick stepper motor drivers.
    \textit{\textbf{Correct.}}
    \end{tcolorbox}
    \end{minipage}
    \caption{
        Case study on \texttt{Industrial-and-Scientific}.
        Both methods receive the same prompt, with SID tokens highlighted in red.
        Outcome GRPO follows the general user interest but misses the target, whereas SAPO identifies more discriminative evidence and predicts the correct SID.
    }
    \label{fig:case-study}
    \vspace{-3mm}
\end{figure}

\subsection{Case Study}
\label{sec:case-study}

To complement the quantitative results, we provide qualitative examples showing how SAPO shapes reasoning behavior.
Figure~\ref{fig:case-study} compares outcome-reward training and SAPO on the same \texttt{Industrial-and-Scientific} test prompt.
GRPO produces a plausible near miss: it captures the broad 3D-printing domain but treats the history as generic printer parts, leading to another accessory.
The key missing cue is the transition from filament and extrusion components to NEMA17 coupling and stepper motor, which indicates motion control.
SAPO preserves this finer-grained cue by reasoning about electronics and stepper motor control, and ranks the motor-driver SID first.

Because both models receive the same prompt and start from the same Stage 2 reasoning-activation checkpoint, the contrast primarily reflects the effect of the Stage 3 credit-assignment objective.
The outcome reward only distinguishes complete SID-tuple success from failure, so this broadly relevant but wrong accessory receives the same scalar penalty as a completely unrelated item.
SAPO instead assigns feedback to reasoning steps using SID-token correctness, making the update sensitive to which SID-token prediction is correct and which refinement fails.
The case study therefore complements the training-dynamics results: reasoning-step credit assignment improves stability while also encouraging reasoning that preserves the evidence needed for fine-grained SID selection.
We provide additional case studies, including examples from \texttt{Video-Games} and \texttt{Office-Products}, in Appendix~\ref{app:case-study}.

%% file: sections/06_conclusion.tex
% =====================================================================
\section{Conclusion}
\label{sec:conclusion}

We introduced \textbf{SAPO}, a step-aligned reinforcement-learning method for reasoning-based generative recommenders that decode items over a $K$-level SID codebook.
The central observation is that outcome-reward RL uses the wrong action granularity for this setting: it obscures which SID-token prediction caused the mismatch and can penalize matched SID-token predictions together with erroneous refinements.
SAPO resolves this mismatch by treating each reasoning step $(\tau^{(k)}, s^{(k)})$ as the action unit and deriving from that choice a per-step match reward, a per-step group-relative advantage, and a step-aligned surrogate with step-normalized token aggregation.
Across the evaluated settings, SAPO materially stabilizes training dynamics and yields competitive-to-improved Recall and NDCG relative to outcome-reward baselines, with the clearest gains where fine-grained credit matters most.
Beyond SIDs, our analysis suggests that whenever a generation task admits a natural hierarchical decomposition, such as a codebook, a tool-call schema, or a structured answer format, the hierarchy itself can serve as a fundamental source of per-step supervision for reinforcement learning, without any learned reward model.

%% file: sections/05_related_work.tex
% =====================================================================
\section{Related Work}
\label{sec:related}

\subsection{Generative Recommendation with Semantic Identifiers}

% [Polish] Original: Generative recommendation has emerged as an alternative to discriminative matching, in which each item is represented as a short sequence of discrete tokens and the recommender is trained as a language model over these tokens~\citep{rajput2023recommender, deng2025onerec, lee2022autoregressive}.
Generative recommendation has emerged as an alternative to discriminative matching: each item is represented as a short sequence of discrete tokens, and the recommender is trained as a language model over these tokens~\citep{rajput2023recommender, deng2025onerec, lee2022autoregressive}.
% [Polish] Original: A central design choice is the tokenization scheme: residual-quantization VAEs produce hierarchical \emph{semantic identifiers} that capture coarse-to-fine item semantics, which both improves generalization to unseen items and enables compact generation~\citep{rajput2023recommender, lee2022autoregressive}.
A central design choice is the tokenization scheme: residual-quantization VAEs produce hierarchical \emph{semantic identifiers} (SIDs) that capture coarse-to-fine item semantics, which both improves generalization to unseen items and enables compact generation~\citep{rajput2023recommender, lee2022autoregressive}.
% [Polish] Original: Our work inherits this representation and shows how the $K$-level structure of SIDs provides a \emph{natural scaffold for step-level credit assignment}, an observation that, to the best of our knowledge, has not been leveraged by prior work on generative recommenders, including recent reasoning-based systems such as SIDReasoner~\citep{he2026reasoning} that still rely on a scalar outcome reward.
Our work inherits this representation and shows how the $K$-level structure of SIDs provides a natural scaffold for step-level credit assignment, an observation that, to the best of our knowledge, has not been leveraged by prior work on generative recommenders, including recent reasoning-based systems such as SIDReasoner~\citep{he2026reasoning} that still rely on a scalar outcome reward.

\subsection{Reasoning and Reinforcement Learning for Large-Language-Model (LLM)-based
Recommendation}

% [Polish] Original: Following the success of reinforcement learning from verifiable rewards on math and code~\citep{guo2025deepseek, lambert2024tulu}, a growing body of work applies group-relative policy-optimization algorithms~\citep{shao2024deepseekmath, yu2025dapo, liu2025understanding, zheng2025group} to LLM-based recommenders.
Following the success of reinforcement learning from verifiable rewards on math and code~\citep{guo2025deepseek, lambert2024tulu}, several recent systems apply group-relative policy-optimization algorithms~\citep{shao2024deepseekmath, yu2025dapo, liu2025understanding, zheng2025group} to LLM-based recommenders.
% [Polish] Original: Most closely related to us, SIDReasoner~\citep{he2026reasoning}, OneRec-Think, R$^2$ec, ReaRec, and RecCoT~\citep{liu2025onerec, you2025r, tang2025think, yang2025reccot} augment generative recommenders with chain-of-thought reasoning and train them with a scalar outcome reward fed into GRPO or one of its close variants.
The systems most closely related to ours include SIDReasoner~\citep{he2026reasoning}, OneRec-Think, R$^2$ec, ReaRec, and RecCoT~\citep{liu2025onerec, you2025r, tang2025think, yang2025reccot}, all of which augment generative recommenders with chain-of-thought reasoning and train them with a scalar outcome reward fed into GRPO or a close variant.
% [Polish] Original: Our analysis in Section~\ref{sec:action-granularity} identifies a structural misalignment between this outcome-reward recipe and $K$-level SID decoding that, empirically, causes these methods to suffer from high-variance gradient estimates whose symptoms include length hacking, drift, and reward collapse, and is independent of which underlying policy-optimization algorithm one chooses.
% \textcolor{red}{Our analysis in Section~\ref{sec:action-granularity} identifies a structural misalignment between this outcome-reward recipe and $K$-level SID decoding: rollout-level credit cannot localize failed SID refinements, which empirically manifests as length drift, degraded format compliance, and reward collapse, independent of the specific policy-optimization algorithm.}
% [Polish] Original: SAPO directly addresses this misalignment by assigning rewards at the per-step granularity, which can be viewed as a form of process-reward RL~\citep{lightman2023let, uesato2022solving} that obtains per-step supervision \emph{for free} from the codebook hierarchy, without any human or learned reward model.
SAPO directly addresses this misalignment by assigning rewards at the per-step granularity; it can be viewed as a form of process-reward RL~\citep{lightman2023let, uesato2022solving} that obtains per-step supervision directly from the codebook hierarchy, without any human-annotated or learned reward model.
% [Polish] Original: Concurrently, Rank-GRPO~\citep{zhu2025rank} proposes rank-level credit assignment for conversational recommenders, and our work is complementary in that it targets the \emph{codebook hierarchy} of SID decoding rather than the \emph{positional hierarchy} of a recommendation list.
Concurrently, Rank-GRPO~\citep{zhu2025rank} proposes rank-level credit assignment for conversational recommenders; our work is complementary in that it targets the \emph{codebook hierarchy} of SID decoding rather than the \emph{positional hierarchy} of a recommendation list.
% [Polish] Original: A complementary diagnosis is provided by~\citet{zhang2026why}, who identify an \emph{inference-time} failure mode of reasoning-based generative recommenders: long chain-of-thought traces induce a textual-inertia bias in which a generic-text subspace dominates inference and dilutes the SID evidence needed to decode the correct item.
A complementary diagnosis is provided by~\citet{zhang2026why}, who identify an \emph{inference-time} failure mode of reasoning-based generative recommenders: long chain-of-thought traces induce a textual-inertia bias in which a generic-text subspace dominates inference and dilutes the SID evidence needed to decode the correct item.
% [Polish] Original: Our analysis targets a different, \emph{training-time} pathology, namely the gradient-variance instability and indiscriminate credit assignment of Section~\ref{sec:action-granularity}, but the two are consistent: an outcome-reward gradient budget indiscriminately spread over reasoning tokens reinforces exactly the reasoning-token distribution whose dominance~\citet{zhang2026why} observe at inference time, whereas SAPO concentrates the budget on the SID decision tokens that the reward can verify.
{Our analysis targets a different, \emph{training-time} failure mode: rollout-level credit diffused over reasoning tokens and SID-token predictions.}
The two diagnoses are nevertheless consistent.
{An outcome-reward gradient budget spread indiscriminately over reasoning tokens reinforces the same reasoning-token distribution whose dominance~\citet{zhang2026why} observe at inference time, whereas SAPO concentrates the budget on the verifiable SID tokens.}

%% file: sections/07_appendix.tex
%%%%%%%%%%%%%%%%%%%%%%%%%%%%%%%%%%%%%%%%%%%%%%%%%%%%%%%%%%%%

\section{Limitations}
\label{app:limitations}

We highlight several directions for extending SAPO beyond the present study.
\emph{Evaluation scale and diversity.}
Our experiments focus on three English Amazon Reviews categories with a fixed $K\!=\!3$ SID hierarchy.
Future work can evaluate SAPO on larger catalogs, deeper or less balanced SID hierarchies, cross-lingual item descriptions, and domains with different interaction dynamics.
It would also be useful to study how step-aligned credit assignment behaves under different initialization regimes, including weaker Stage~2 checkpoints.
\emph{Richer feedback settings.}
This paper uses ground-truth SID tuples to obtain clean per-level verifiable feedback.
A natural next step is to extend SAPO to settings with multiple acceptable next items, noisy implicit feedback, or production-oriented click and conversion signals.
These settings may require combining SID-level match rewards with calibrated relevance, diversity, or user-satisfaction objectives.
\emph{System and decoding variants.}
Our implementation uses a blocked layout in which the policy generates all $K$ thinking blocks before decoding the contiguous SID block.
Future work can explore alternative response layouts, decoding constraints, and reward schedules, including interleaved reasoning--SID generation, adaptive step weighting, and deeper integration with trie-constrained decoding.
These extensions may further improve efficiency, controllability, and deployment flexibility.

\section{Broader Impacts}
\label{app:broader-impacts}

SAPO improves the training of generative recommenders that decode items as hierarchical SIDs.
On the positive side, per-step credit assignment reduces the tendency of outcome-reward training to inflate reasoning length and collapse reward, which translates into more sample-efficient RL fine-tuning and more interpretable (coarse-to-fine) reasoning traces that practitioners can audit.
On the negative side, any improvement to the accuracy of personalized recommenders carries the standard risks of the broader generative-recommendation setting: filter-bubble reinforcement, preference manipulation by adversarial content producers, and information asymmetry between platforms and users.
Our contribution targets the RL credit-assignment layer and does not introduce new data sources or deployment pathways beyond those already common in the literature, so we do not identify a novel direct path to harm; we nonetheless encourage downstream practitioners who build on SAPO to pair it with diversity-aware reward shaping and with the auditing tools proposed in the fairness and filter-bubble literature.

\section{Theoretical Proofs and Discussions}
\label{app:proofs}

This appendix formalizes the objective-consistency claim used in Section~\ref{sec:reward-design}.
Appendix~\ref{app:proofs-lcp} proves that the cumulative match component of SAPO's per-step reward equals $\alpha$ times the number of matched SID positions, preserves the exact-match optimum under realizability, and becomes an aligned marginal refinement when exact matching is not realizable.
Appendix~\ref{app:decoding-order} formalizes the implementation decoding order in Eq.~\eqref{eq:response-impl} and explains why the paired unit $(\tau^{(k)},s^{(k)})$ remains the relevant credit-assignment unit.
These results do not modify the SAPO method of Section~\ref{sec:method}; they only formalize the consistency claim and the implementation details behind step-aligned credit assignment.

\subsection{Formal Statement and Proof of the Consistency Claim}
\label{app:proofs-lcp}

For a rollout $y_i$, let
\begin{equation}
\phi(\hat v_i)=(s_i^{(1)},\ldots,s_i^{(K)}),
\qquad
\phi(v_i^\star)=(s_i^{\mathrm{gt},(1)},\ldots,s_i^{\mathrm{gt},(K)})
\end{equation}
denote the generated and ground-truth SID tuples.
Define the number of matched SID positions as
\begin{equation}
\#\mathrm{Match}\big(\phi(\hat v_i),\phi(v_i^\star)\big)
\coloneqq
\sum_{k=1}^K
\mathbbm{1}\!\left[s_i^{(k)}=s_i^{\mathrm{gt},(k)}\right].
\end{equation}
Recall that SAPO's per-step reward in Eq.~\eqref{eq:step-reward} contains a match component
\begin{equation}
m_{i,k}
\coloneqq
\alpha\,\mathbbm{1}\!\left[s_i^{(k)}=s_i^{\mathrm{gt},(k)}\right].
\end{equation}
We analyze this match component separately from the final-step format bonus.

\begin{proposition}[Per-step match rewards and outcome optima]
\label{prop:lcp}
Let $\Pi$ be a policy class over rollouts, and define
\begin{equation}
J_{\mathrm{out}}(\pi)
\coloneqq
\mathbb{E}_{x,y\sim\pi(\cdot\mid x)}
\left[
\mathbbm{1}\!\left(\phi(\hat v)=\phi(v^\star)\right)
\right],
\end{equation}
and
\begin{equation}
J_{\mathrm{match}}(\pi)
\coloneqq
\mathbb{E}_{x,y\sim\pi(\cdot\mid x)}
\left[
\sum_{k=1}^K m_k
\right].
\end{equation}
Then the following statements hold.

\begin{enumerate}[leftmargin=*,itemsep=2pt,topsep=2pt]
    \item \textbf{Identity.}
    For every rollout $y_i$,
    \begin{equation}
    \sum_{k=1}^K m_{i,k}
    =
    \alpha\,
    \#\mathrm{Match}\big(\phi(\hat v_i),\phi(v_i^\star)\big).
    \end{equation}

    \item \textbf{Optimum preservation under realizability.}
    Suppose there exists a policy $\pi^\star\in\Pi$ such that
    \begin{equation}
    \Pr_{y\sim\pi^\star(\cdot\mid x)}
    \left[
    \phi(\hat v)=\phi(v^\star)
    \right]
    =
    1
    \qquad
    \text{for every } x\in\mathcal{D}_{\mathrm{RL}}.
    \end{equation}
    Then
    \begin{equation}
    \arg\max_{\pi\in\Pi} J_{\mathrm{out}}(\pi)
    =
    \arg\max_{\pi\in\Pi} J_{\mathrm{match}}(\pi).
    \end{equation}

    \item \textbf{Aligned refinement under non-realizability.}
    Without the realizability assumption,
    \begin{equation}
    J_{\mathrm{match}}(\pi)
    =
    \alpha
    \sum_{k=1}^K R_k(\pi),
    \end{equation}
    where
    \begin{equation}
    R_k(\pi)
    \coloneqq
    \Pr_{x,y\sim\pi(\cdot\mid x)}
    \left[
    s^{(k)}=s^{\mathrm{gt},(k)}
    \right]
    \end{equation}
    is the level-$k$ marginal match probability.
    By contrast,
    \begin{equation}
    J_{\mathrm{out}}(\pi)
    =
    \Pr_{x,y\sim\pi(\cdot\mid x)}
    \left[
    \bigwedge_{k=1}^K
    s^{(k)}=s^{\mathrm{gt},(k)}
    \right],
    \end{equation}
    the joint probability that all SID positions match.
    Thus maximizing $J_{\mathrm{match}}$ maximizes the uniform sum of level-wise marginal match probabilities, providing an aligned refinement of the single joint exact-match objective.
\end{enumerate}
\end{proposition}

\begin{proof}
\textbf{Part 1.}
By definition,
\begin{equation}
\frac{m_{i,k}}{\alpha}
=
\mathbbm{1}\!\left[s_i^{(k)}=s_i^{\mathrm{gt},(k)}\right].
\end{equation}
Summing over $k$ gives the number of SID positions on which the generated and ground-truth tuples agree:
\begin{equation}
\sum_{k=1}^K \frac{m_{i,k}}{\alpha}
=
\sum_{k=1}^K
\mathbbm{1}\!\left[s_i^{(k)}=s_i^{\mathrm{gt},(k)}\right]
=
\#\mathrm{Match}\big(\phi(\hat v_i),\phi(v_i^\star)\big).
\end{equation}
Multiplying both sides by $\alpha$ proves the identity.

\textbf{Part 2.}
For any policy $\pi$, the outcome objective satisfies
\begin{equation}
0\le J_{\mathrm{out}}(\pi)\le 1.
\end{equation}
Under the realizability assumption, the upper bound is attainable. Therefore, a policy maximizes $J_{\mathrm{out}}$ if and only if it generates the ground-truth SID tuple almost surely for every prompt.

For the cumulative match reward, Part~1 implies
\begin{equation}
0
\le
\sum_{k=1}^K m_k
\le
K\alpha.
\end{equation}
The upper bound $K\alpha$ is attained for a rollout if and only if
\begin{equation}
s^{(k)}=s^{\mathrm{gt},(k)}
\qquad
\text{for all } k\in[K],
\end{equation}
which is equivalent to
\begin{equation}
\phi(\hat v)=\phi(v^\star).
\end{equation}
Since the quantizer $\phi$ is injective on the catalog, this is also equivalent to exact item recovery.
Thus a policy maximizes $J_{\mathrm{match}}$ if and only if it generates the ground-truth SID tuple almost surely for every prompt.
The maximizers of $J_{\mathrm{out}}$ and $J_{\mathrm{match}}$ are therefore identical over $\Pi$.

\textbf{Part 3.}
By Part~1 and linearity of expectation,
\begin{equation}
J_{\mathrm{match}}(\pi)
=
\mathbb{E}_{x,y\sim\pi(\cdot\mid x)}
\left[
\sum_{k=1}^K m_k
\right]
=
\alpha
\sum_{k=1}^K
\Pr_{x,y\sim\pi(\cdot\mid x)}
\left[
s^{(k)}=s^{\mathrm{gt},(k)}
\right]
=
\alpha\sum_{k=1}^K R_k(\pi).
\end{equation}
In contrast,
\begin{equation}
J_{\mathrm{out}}(\pi)
=
\Pr_{x,y\sim\pi(\cdot\mid x)}
\left[
\bigwedge_{k=1}^K
s^{(k)}=s^{\mathrm{gt},(k)}
\right].
\end{equation}
This joint exact-match probability is bounded above by each marginal:
\begin{equation}
J_{\mathrm{out}}(\pi)
\le
R_k(\pi),
\qquad
\forall k\in[K],
\end{equation}
and hence
\begin{equation}
J_{\mathrm{out}}(\pi)\le \min_{k\in[K]}R_k(\pi).
\end{equation}
Therefore, two policies can have the same exact-match probability while differing in their level-wise marginals.
The cumulative match reward distinguishes such policies by assigning credit to matched SID positions, and is therefore an aligned marginal refinement of the outcome objective.
\end{proof}

\noindent\textbf{Remark (informative rollout groups).}
The same distinction appears in the group-relative advantages used during RL.
Outcome-reward GRPO produces nonzero group-relative contrast only when the exact-match indicators differ across rollouts in the same group.
Equivalently, the group must contain at least two rollouts with different values of
\begin{equation}
\mathbbm{1}\!\left[\phi(\hat v)=\phi(v^\star)\right].
\end{equation}
SAPO instead computes reasoning-step advantages from the level-wise indicators
\begin{equation}
\mathbbm{1}\!\left[s^{(k)}=s^{\mathrm{gt},(k)}\right],
\qquad k\in[K].
\end{equation}
Thus, for a fixed level $k$, SAPO has nonzero step-level contrast whenever these level-$k$ indicators differ across rollouts.
This condition is strictly weaker than requiring different exact-match indicators.
For example, if all rollouts miss the final item but some match the first SID level while others do not, the outcome advantage vanishes for the whole group, whereas SAPO still obtains a nonzero advantage at level~1.
More generally, whenever sampled rollouts disagree on any SID position, SAPO can provide a step-level training signal even if every rollout is an exact-match failure.
Therefore, step-aligned advantages are informative on a strictly larger set of rollout groups than outcome-level advantages.

\section{Implementation Layout and Reasoning-Step Pairing}
\label{app:decoding-order}

The policy we actually train and release follows the implementation layout of Eq.~\eqref{eq:response-impl}, in which all $K$ thinking blocks are generated first and the $K$ SID tokens are decoded as a contiguous block.
Under this layout, SAPO defines reasoning step~$k$ by pairing the $k$-th thinking block $\tau^{(k)}$ with the $k$-th SID token $s^{(k)}$.
This pairing is used only to assign credit; it is not an alternative autoregressive factorization of the model.

The distinction matters because under Eq.~\eqref{eq:response-impl} every $s^{(k)}$ is conditioned on the full reasoning trace, i.e., all $K$ thinking blocks, and the previously generated SID tokens, not merely on a local step fragment.
We adopt the blocked layout because (a)~the contiguous SID block enables trie-based constrained decoding and beam search directly over the $K$-level codebook while reusing the KV cache across SID positions, and (b)~later thinking blocks can refine earlier ones before any SID is committed.

SAPO therefore never relies on an alternative decode order.
Per-step match rewards are placed on SID-token positions (Section~\ref{sec:reward-design}), and per-step advantages are computed from per-step group statistics (Section~\ref{sec:step-aligned-grpo}); both depend only on the extracted reasoning steps $(\tau^{(k)}, s^{(k)})$ and the corresponding level-$k$ match indicators.
All definitions and derivations in Sections~\ref{sec:action-granularity}--\ref{sec:step-aligned-grpo} and in Appendix~\ref{app:proofs} therefore remain valid under the implementation layout, with $|y^{(k)}|=|\tau^{(k)}|+1$ denoting the size of the reasoning step.

\section{Baseline Descriptions}
\label{app:baselines}

% We briefly describe each baseline used in Section~\ref{sec:baselines}.

\noindent\textbf{Traditional sequential recommenders.}
\textbf{GRU4Rec}~\citep{hidasi2015session} models user sessions with gated recurrent units (GRUs) and ranks items by session-parallel mini-batch training.
\textbf{SASRec}~\citep{kang2018self} applies a left-to-right Transformer to the interaction sequence, predicting the next item with a causal self-attention mechanism.
% \textbf{BERT4Rec}~\citep{sun2019bert4rec} uses a bidirectional Transformer with a cloze-style (masked-item) training objective, enabling it to attend to both past and future context during training.
\textbf{Caser}~\citep{tang2018personalized} treats the interaction sequence as a matrix and applies horizontal and vertical convolutional filters to capture point-level and union-level sequential patterns.

\noindent\textbf{Generative recommenders without reasoning.}
\textbf{TIGER}~\citep{rajput2023recommender} represents items as sequences of SIDs produced by an RQ-VAE and trains a Transformer to autoregressively generate SID tuples for the next item.
\textbf{HSTU}~\citep{zhai2024actions} is a large-scale generative retrieval model that encodes heterogeneous user actions with a hierarchical sequence Transformer.
\textbf{LCRec}~\citep{zheng2024adapting} aligns a pre-trained language model with collaborative item IDs via a lightweight adapter, enabling generative recommendation without full fine-tuning.
\textbf{LETTER}~\citep{wang2024learnable} learns item tokenization end-to-end with the generative recommender, jointly optimizing the codebook and the decoder.
% \textbf{CoRAB}~\citep{TODO-corab} introduces a contrastive regularization strategy to improve SID-based generative recommendation by reducing codebook collision.

\noindent\textbf{Reasoning-based generative recommenders.}
\textbf{SIDReasoner}~\citep{he2026reasoning} augments a generative recommender with chain-of-thought reasoning and trains Stage~3 with outcome-reward RL; it shares our Stages~1 and~2.
\textbf{OneRec-Think}~\citep{liu2025onerec} extends the OneRec framework with a reasoning module and outcome-reward GRPO training.
\textbf{R$^2$ec}~\citep{you2025r} integrates retrieval-augmented reasoning into generative recommendation and trains with a scalar reward signal.
\textbf{ReaRec}~\citep{tang2025think} uses a think-then-recommend paradigm with reasoning traces distilled from a teacher model, followed by outcome-reward RL fine-tuning.

\section{Dataset Details}
\label{app:data}

\noindent\textbf{Source and preprocessing.}
We use the Amazon Reviews corpus~\citep{ni2019justifying} and follow the 5-core sequential recommendation protocol of~\citet{kang2018self}: users and items with fewer than five interactions are filtered out, and the remaining interactions are sorted chronologically per user.
We restrict the timespan to October~2016--November~2018 so that all three categories share a comparable activity window.
Each user's interaction sequence is split leave-one-out: the last item is held out for test, the second-to-last for validation, and the remaining prefix for training (with all valid (history, target) sliding-window pairs used as training samples).

\noindent\textbf{Per-category statistics.}
Table~\ref{tab:dataset-stats} reports the final split sizes used throughout the paper.
The three categories span a range of catalog sizes (3.5k--3.9k items) and activity densities (avg.\ sequence length 5.96--6.45), providing a consistent but non-trivial test bed for generative recommendation.

\begin{table}[h]
\centering
\caption{Dataset statistics for the three Amazon Reviews categories after 5-core filtering and leave-one-out splitting. `\#Train/\#Valid/\#Test' count (history, target) pairs; `Avg.\ seq.\ len.' is the average number of observed interactions per test user (including the test target).}
\label{tab:dataset-stats}
\small
\begin{tabular}{@{}l rrrrrr@{}}
\toprule
\textbf{Category} & \#Users & \#Items & \#Train & \#Valid & \#Test & Avg.\ seq.\ len. \\
\midrule
\texttt{Office-Products}           & 4{,}866 & 3{,}459 & 38{,}924 & 4{,}866 & 4{,}866 & 5.97 \\
\texttt{Video-Games}               & 6{,}142 & 3{,}858 & 49{,}133 & 6{,}142 & 6{,}142 & 6.45 \\
\texttt{Industrial-and-Scientific} & 4{,}533 & 3{,}686 & 36{,}259 & 4{,}532 & 4{,}533 & 5.96 \\
\bottomrule
\end{tabular}
\end{table}

\noindent\textbf{SID tokenization (RQ-VAE).}
For each category we learn a residual-quantization variational autoencoder (RQ-VAE) over item text embeddings concatenating title, category, brand, and description.
The quantizer has $K{=}3$ codebook levels with 256 entries per level, yielding a 3-tuple SID per item of the form $\langle a_{i_1}\rangle\langle b_{i_2}\rangle\langle c_{i_3}\rangle$ with $i_k\in\{0,\dots,255\}$.
The RQ-VAE is trained per category with mean squared error (MSE) reconstruction loss and commitment loss until all item embeddings receive a unique SID tuple; because $K{=}3$ levels at $256^3$ codes far exceed the catalog size, a single pass suffices.

\noindent\textbf{Codebook utilization.}
The effective codebook utilization (fraction of the 256 entries at each level that appear in at least one item's SID) is reported in Table~\ref{tab:codebook-util}.
Level~1 is sparsely used (19--34\%), reflecting the small number of broad categorical clusters in each dataset, whereas levels~2 and~3 are fully utilized.
{This pattern motivates SAPO's per-step credit assignment: since the prediction difficulty per step is highly asymmetric, a single scalar advantage cannot reflect which SID-token position the policy gets right or wrong.}

\begin{table}[h]
\centering
\caption{Codebook utilization per level across the three categories (\# unique codes observed among 256 entries).}
\label{tab:codebook-util}
\small
\begin{tabular}{@{}l ccc@{}}
\toprule
\textbf{Category} & Level~1 ($a$) & Level~2 ($b$) & Level~3 ($c$) \\
\midrule
\texttt{Office-Products}           & 88 & 256 & 256 \\
\texttt{Video-Games}               & 60 & 256 & 256 \\
\texttt{Industrial-and-Scientific} & 48 & 256 & 256 \\
\bottomrule
\end{tabular}
\end{table}

\section{Dataset and LLM Input Examples}
\label{app:data-examples}

This appendix complements Appendix~\ref{app:data} by providing concrete examples of the raw item/user/interaction records, the learned $K{=}3$ SID tuples, and the natural-language prompt template used by the LLM at training and inference time.

\subsection{Per-Category Dataset Examples}
\label{app:data-examples-records}

Below we show representative records for each of the three Amazon Reviews categories.
All values are drawn from the filtered and RQ-VAE-tokenized corpus described in Appendix~\ref{app:data}.

\noindent\textbf{Office-Products.}
\begin{itemize}[leftmargin=*,itemsep=0pt,topsep=2pt]
    \item \emph{Item example}: \texttt{item\_id=B000GHJL2C}, \texttt{title}=``Sanford Uni-Ball Vision Needle Micro Rollerball Pen'', \texttt{category}=``Office Products $\to$ Pens \& Writing $\to$ Pens'', \texttt{brand}=``Uni-ball'', \texttt{description}=``Smooth, ultra-fine 0.5\,mm tip; archival-quality pigment ink''.
    \item \emph{SID encoding}: $(s^{(1)}, s^{(2)}, s^{(3)}) = (\langle a_{17}\rangle,\langle b_{93}\rangle,\langle c_{41}\rangle)$ (3-level RQ-VAE over text embeddings of title + category + brand + description).
    \item \emph{User example}: \texttt{user\_id=A1P\ldots{}XQ}, \#interactions$=8$ (after 5-core filtering).
    \item \emph{Interaction example (test record)}: \texttt{user\_id=A1P\ldots{}XQ}, \texttt{item\_sid}=\texttt{<a\_17><b\_93><c\_41>}, \texttt{history\_item\_sid}=\texttt{[<a\_3><b\_12><c\_201>, <a\_17><b\_88><c\_14>, <a\_5><b\_220><c\_77>]}, \texttt{timestamp}=\texttt{2018-09-24}.
\end{itemize}

\noindent\textbf{Video-Games.}
\begin{itemize}[leftmargin=*,itemsep=0pt,topsep=2pt]
    \item \emph{Item example}: \texttt{item\_id=B00E46CV88}, \texttt{title}=``Nintendo Selects: Super Mario Galaxy'', \texttt{category}=``Video Games $\to$ Legacy Systems $\to$ Nintendo Systems $\to$ Wii'', \texttt{brand}=``Nintendo'', \texttt{description}=``Platforming adventure featuring spherical worlds and gravity-based puzzles''.
    \item \emph{SID encoding}: $(s^{(1)}, s^{(2)}, s^{(3)}) = (\langle a_{6}\rangle,\langle b_{145}\rangle,\langle c_{182}\rangle)$.
    \item \emph{User example}: \texttt{user\_id=A2K\ldots{}WB}, \#interactions$=11$.
    \item \emph{Interaction example (test record)}: \texttt{user\_id=A2K\ldots{}WB}, \texttt{item\_sid}=\texttt{<a\_6><b\_145><c\_182>}, \texttt{history\_item\_sid}=\texttt{[<a\_6><b\_210><c\_35>, <a\_2><b\_102><c\_19>, <a\_6><b\_145><c\_248>, <a\_31><b\_7><c\_60>]}, \texttt{timestamp}=\texttt{2017-12-02}.
\end{itemize}

\noindent\textbf{Industrial-and-Scientific.}
\begin{itemize}[leftmargin=*,itemsep=0pt,topsep=2pt]
    \item \emph{Item example}: \texttt{item\_id=B003U6XQ12}, \texttt{title}=``Fluke 117 Electricians True RMS Multimeter'', \texttt{category}=``Industrial \& Scientific $\to$ Test \& Measurement $\to$ Multimeters'', \texttt{brand}=``Fluke'', \texttt{description}=``True-RMS voltage and current readings with non-contact voltage detection''.
    \item \emph{SID encoding}: $(s^{(1)}, s^{(2)}, s^{(3)}) = (\langle a_{22}\rangle,\langle b_{4}\rangle,\langle c_{118}\rangle)$.
    \item \emph{User example}: \texttt{user\_id=A9R\ldots{}LN}, \#interactions$=6$.
    \item \emph{Interaction example (test record)}: \texttt{user\_id=A9R\ldots{}LN}, \texttt{item\_sid}=\texttt{<a\_22><b\_4><c\_118>}, \texttt{history\_item\_sid}=\texttt{[<a\_22><b\_4><c\_31>, <a\_13><b\_198><c\_9>, <a\_22><b\_55><c\_240>]}, \texttt{timestamp}=\texttt{2018-04-15}.
\end{itemize}

The history and target are both expressed as SID tuples so that the entire retrieval/ranking objective reduces to next-item SID prediction; the item text fields are only used during SID tokenization and for human inspection, never shown to the LLM.

\subsection{Prompt Template and Instantiated Example}
\label{app:data-examples-prompts}

The LLM receives a two-message chat prompt: a system instruction that frames the task and a user message containing the SID-valued interaction history.
Figure~\ref{fig:prompt-template} shows the template with placeholders, and Figure~\ref{fig:prompt-example} instantiates it with a concrete \texttt{Office-Products} record together with a representative SAPO rollout following the implementation layout of Eq.~\eqref{eq:response-impl}: $K{=}3$ contiguous thinking blocks, one per codebook level, followed by the $K{=}3$ SID tokens under trie-based constrained decoding.

\begin{figure}[t]
\centering
\fbox{\begin{minipage}{0.94\linewidth}
\small
\textbf{Prompt Template (Stage 3 RL rollout).}\\[2pt]
\textbf{System:} \textit{Below is an instruction that describes a task, paired with an input that provides further context. Write a response that appropriately completes the request. Can you recommend the next item for the user based on their interaction history?}\\[2pt]
\textbf{User:} The user has sequentially interacted with items $\langle$\texttt{history\_sid\_list}$\rangle$. Can you recommend the next item for him? Let's think step by step before making recommendation. Directly output the item SID after thinking.\\[4pt]
\textbf{Expected assistant response:}\\
\texttt{<think>} $\langle$\texttt{reasoning trace 1}$\rangle$ \texttt{</think>} \texttt{<think>} $\langle$\texttt{reasoning trace 2}$\rangle$ \texttt{</think>} \texttt{<think>} $\langle$\texttt{reasoning trace 3}$\rangle$ \texttt{</think>} $\langle$\texttt{SID level 1}$\rangle$$\langle$\texttt{SID level 2}$\rangle$$\langle$\texttt{SID level 3}$\rangle$
\end{minipage}}
\caption{Prompt template used for Stage~3 RL training. Placeholders in $\langle\cdot\rangle$ are filled with per-example content: $\langle$\texttt{history\_sid\_list}$\rangle$ is a comma-separated list of SID tuples drawn from the user's interaction history, and the expected response is $K{=}3$ \texttt{<think>}$\ldots$\texttt{</think>} thinking blocks (one per codebook level, matching the decoding layout of Eq.~\eqref{eq:response-impl}) followed by the $K{=}3$ SID tokens for the predicted next item. Constrained decoding over the learned SID trie guarantees that every sampled SID is a valid item under the per-category codebook.}
\label{fig:prompt-template}
\end{figure}

\begin{figure}[t]
\centering
\fbox{\begin{minipage}{0.94\linewidth}
\small
\textbf{Prompt Example (\texttt{Office-Products}).}\\[2pt]
\textbf{System:} \textit{Below is an instruction that describes a task, paired with an input that provides further context. Write a response that appropriately completes the request. Can you recommend the next item for the user based on their interaction history?}\\[2pt]
\textbf{User:} The user has sequentially interacted with items \texttt{<a\_3><b\_12><c\_201>, <a\_17><b\_88><c\_14>, <a\_5><b\_220><c\_77>}. Can you recommend the next item for him? Let's think step by step before making recommendation. Directly output the item SID after thinking.\\[4pt]
\textbf{Ground-truth next item:} \texttt{<a\_17><b\_93><c\_41>} (Sanford Uni-Ball Vision Needle Micro Rollerball Pen).\\[2pt]
\textbf{Example SAPO rollout (full-match case):}\\
\texttt{<think>} The user's recent purchases cluster around general office writing supplies, with a strong signal from \texttt{<a\_17>} (an office-supplies coarse code that appears in the second history item). The fine-grained codes span pens and similar stationery, so the next item likely sits under the same level-1 cluster. \texttt{</think>} \texttt{<think>} Narrowing to level~2, the user has interacted with sub-code \texttt{<b\_88>} (ballpoint writing); a nearby sub-code around \texttt{<b\_93>} covers rollerball/gel-ink pens, which is a natural next purchase given the existing ballpoint history. \texttt{</think>} \texttt{<think>} Within that level-2 cluster, level-3 variants around \texttt{<c\_14>} to \texttt{<c\_45>} correspond to fine-point rollerball variants; \texttt{<c\_41>} specifically matches a smooth ultra-fine needle-tip rollerball, consistent with the item description style of the user's prior purchases. \texttt{</think>} \texttt{<a\_17><b\_93><c\_41>}
\end{minipage}}
\caption{A concrete prompt instance from \texttt{Office-Products}. The user's history is represented entirely as SID tuples, and the rollout generates $K{=}3$ \texttt{<think>}$\ldots$\texttt{</think>} thinking blocks (one per codebook level, matching the decoding layout of Eq.~\eqref{eq:response-impl}) followed by the $K{=}3$ SID tokens. In this example the rollout recovers the ground-truth target exactly (all three levels correct); see Section~\ref{sec:case-study} and Appendix~\ref{app:case-study} for side-by-side comparisons with outcome-reward rollouts under the same prompt. The per-step match reward of Eq.~\eqref{eq:step-reward} is non-zero at every codebook level for this rollout (all three per-step match indicators fire), which is exactly the per-step match signal SAPO exploits.}
\label{fig:prompt-example}
\end{figure}

\section{Implementation Details}
\label{app:implementation}

\textbf{Default backbone.}
We use Qwen3-1.7B~\citep{yang2025qwen3} as the primary backbone and share the Stage~1 and Stage~2 checkpoints across all RL variants so that every comparison isolates the effect of the Stage~3 algorithm.
Stages~1 and~2 are trained with AdamW, a cosine schedule, and a $5\%$ warm-up on the enriched SFT and reasoning-activation datasets described in Appendix~\ref{app:training-details}.
In Stage~3 we build on the VERL~\citep{sheng2025hybridflow} training framework with vLLM-based rollouts.
For each prompt we sample $G\!=\!16$ rollouts with temperature and top-$p$ both set to $1.0$, a maximum response length of $1024$ tokens, and a mini-batch of $256$ prompts per update.
The actor learning rate is $1\!\times\!10^{-6}$ with a $2\%$ warm-up, the KL coefficient against the Stage~2 reference is set to $10^{-3}$ (low-variance KL), the entropy coefficient is $0$, and the clip range $\epsilon$ is $0.2$ for both SAPO and the outcome-reward baselines.
The released SAPO implementation assigns the per-step match reward of Eq.~\eqref{eq:step-reward} at SID-token positions, attaches the format bonus only to the last SID position, and requires exactly $K$ extracted thinking blocks before any per-step reward is credited.

% \textbf{Additional backbones.}
% To assess the generalizability of SAPO across model scales and architectures, we additionally run the full three-stage pipeline with Qwen3-4B and Llama-3.2-3B.
% For each backbone we retrain Stages~1--3 from scratch using the same data splits and hyperparameter search space (Appendix~\ref{app:training-details}); only the learning rate and batch size are re-tuned per model.

\textbf{Reproducibility.}
All experiments run on $4$ NVIDIA GPUs.
% Each configuration is run with $3$ random seeds and we report the mean and standard deviation.
Full hyperparameters, and compute budgets are provided in Appendix~\ref{app:training-details}.

\section{Training Details}
\label{app:training-details}

Stages~1--2 use the HuggingFace \texttt{Trainer}, and Stage~3 builds on VERL~\citep{sheng2025hybridflow} with vLLM-based rollouts.
% We fix seed~$42$ globally.
Table~\ref{tab:hparams} summarizes the key hyperparameters per stage for the default Qwen3-1.7B backbone. 
% Table~\ref{tab:per-backbone} lists the re-tuned values used for the additional Qwen3-4B and Llama-3.2-3B backbones.

\begin{table}[h]
\centering
\caption{Hyperparameters per training stage for the Qwen3-1.7B backbone. All stages use AdamW in BF16 with $4$ GPUs. Early stopping on validation loss (patience~$3$) is active in Stages~1--2. Stage~3 follows the group-relative surrogate of Eq.~\eqref{eq:step-aligned-obj} with the rollout parameters shown.}
\label{tab:hparams}
\footnotesize
\setlength{\tabcolsep}{3pt}
\begin{tabular}{@{}l l l l@{}}
\toprule
\textbf{Hyperparameter} & \textbf{Stage~1 (SFT)} & \textbf{Stage~2 (reasoning)} & \textbf{Stage~3 (RL)} \\
\midrule
Initial checkpoint        & Qwen3-1.7B~\citep{yang2025qwen3} & Stage~1 final  & Stage~2 final \\
Optimizer                  & AdamW            & AdamW             & AdamW           \\
Learning rate              & $3\!\times\!10^{-4}$ & $1\!\times\!10^{-5}$ & $1\!\times\!10^{-6}$ \\
LR schedule                & cosine           & cosine            & cosine          \\
Warm-up                    & $20$ steps       & $10$ steps        & $2\%$ of total  \\
Weight decay               & $0.0$            & $0.0$             & $0.0$           \\
Global batch size          & $1024$ samples   & $1024$ samples    & $256$ prompts   \\
Per-GPU micro-batch        & $1$              & $8$               & $8$             \\
Max epochs                 & $10$             & $5$               & $10$            \\
Max sequence length        & $1024$ tokens    & $1024$ tokens     & $1024{+}1024$   \\
Precision                  & BF16             & BF16              & FP32            \\
\midrule
\multicolumn{4}{@{}l}{\textit{Stage~3 RL-specific.}}\\
Rollouts per prompt $G$    & \multicolumn{3}{l}{$16$ (vLLM, temperature $1.0$, top-$p\!=\!1.0$, GPU-mem util.\ $0.8$)} \\
KL coefficient $\lambda$   & \multicolumn{3}{l}{$10^{-3}$ (low-variance KL to the Stage~2 reference)} \\
Entropy coefficient        & \multicolumn{3}{l}{$0$} \\
Clip range $\epsilon$      & \multicolumn{3}{l}{$0.2$ (SAPO and outcome-reward baselines)} \\
Reward coefficients        & \multicolumn{3}{l}{$\alpha\!=\!1.0$ (match), $\beta\!=\!0.2$ (format)} \\
Decoding constraint        & \multicolumn{3}{l}{trie-based constrained decoding over the $K$-level SID vocabulary} \\
\bottomrule
\end{tabular}
\end{table}

% \begin{table}[h]
% \centering
% \caption{Backbone-specific overrides for Qwen3-4B and Llama-3.2-3B. Values not listed match Table~\ref{tab:hparams}. The final tuned values for the larger backbones will be filled in alongside the future RQ5 results; learning rates are searched over $\{5\!\times\!10^{-7}, 1\!\times\!10^{-6}, 2\!\times\!10^{-6}\}$ for Stage~3 and scaled commensurately in Stages~1--2, with per-device batch sizes halved (and accumulation doubled) to keep within the same $4\!\times$ GPU budget.}
% \label{tab:per-backbone}
% \small
% \begin{tabular}{@{}l ccc@{}}
% \toprule
% \textbf{Setting} & \textbf{Qwen3-1.7B (default)} & \textbf{Qwen3-4B} & \textbf{Llama-3.2-3B} \\
% \midrule
% Stage~1 LR                   & $3\!\times\!10^{-4}$ & -- & -- \\
% Stage~2 LR                   & $1\!\times\!10^{-5}$ & -- & -- \\
% Stage~3 actor LR             & $1\!\times\!10^{-6}$ & -- & -- \\
% Stage~1--2 global batch      & $1024$               & -- & -- \\
% Stage~3 global batch         & $256$                & -- & -- \\
% Stage~3 rollouts per prompt  & $16$                 & $16$ & $16$ \\
% \bottomrule
% \end{tabular}
% \end{table}

\noindent\textbf{Compute budget.}
On $4\!\times\!$NVIDIA~H100 (80\,GB) GPUs, end-to-end training of the default Qwen3-1.7B pipeline takes approximately $20$~hours for Stage~1, $1$~hours for Stage~2, and $48$~hours for Stage~3 per category (wall-clock).
Stage~3 dominates total cost and scales roughly linearly with the number of rollouts per prompt.
% Timings for the Qwen3-4B and Llama-3.2-3B backbones will be reported together with the corresponding future RQ5 recommendation metrics.

% \noindent\textbf{Reproduction commands.}
% The three stages correspond to launcher scripts \texttt{sft\_Qwen3\_enrich.sh}, \texttt{sft\_reasoning\_activation\_step\_aligned.sh}, and \texttt{RL\_training\_step\_aligned.sh} in the released code; each takes a single argument selecting one of the three Amazon categories (\texttt{Office\_Products}, \texttt{Video\_Games}, \texttt{Industrial\_and\_Scientific}).
% Outcome-reward baselines are obtained by substituting \texttt{RL\_training\_script.sh} (advantage estimator \texttt{grpo}, reward manager \texttt{naive}) in place of the step-aligned Stage~3 script while keeping Stages~1--2 identical.

\section{Training Efficiency}
\label{app:efficiency}

We report the computational overhead of SAPO relative to outcome-reward training in Table~\ref{tab:efficiency}.
The per-step reward computation and advantage normalization add minimal overhead because they operate on the $K{=}3$ SID positions per rollout rather than on all tokens.
We measure wall-clock time over a five-step profiling run and report averages over steps 2--5 to exclude initialization and first-step warm-up effects.

\begin{table}[h]
\centering
\caption{Training efficiency comparison on \texttt{Office-Products} (Qwen3-1.7B, 4$\times$ GPU). Peak memory denotes maximum allocated GPU memory.}
\label{tab:efficiency}
\begin{tabular}{@{}l ccc@{}}
\toprule
\textbf{Method} & Wall-clock / step (s) & Peak GPU mem.\ (GB) & Reward+adv.\ / step (s) \\
\midrule
Outcome & 85.15 & 74.83 & 0.59 \\
SAPO    & 85.03 & 74.89 & 1.13 \\
\bottomrule
\end{tabular}
\end{table}

SAPO introduces no measurable wall-clock slowdown in this short profiling run: its average step time is within $0.15\%$ of the outcome-reward variant.
The peak allocated GPU memory is also essentially unchanged (74.89\,GB vs.\ 74.83\,GB).
The reward and advantage computation time increases from 0.59\,s to 1.13\,s per step, but the additional 0.54\,s accounts for only $0.63\%$ of a full update step, so it is dominated by rollout generation and actor updates.
We therefore report the directly measured reward/advantage overhead rather than steps to convergence, since convergence in stochastic GRPO training is not a well-defined quantity for a short profiling run.

\begin{figure}[h]
    \centering
    \includegraphics[width=\linewidth]{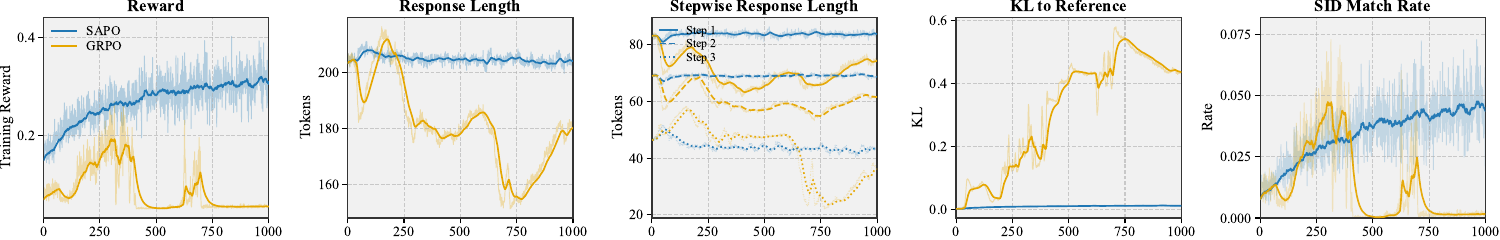}
    \caption{RL training dynamics on the \texttt{Video-Games} dataset.
    This figure follows the same diagnostic layout as Figure~\ref{fig:diagnostic}, comparing SAPO and GRPO from the same Stage~2 checkpoint using training reward, response length, stepwise response length, KL to the reference, and SID match rate; blue denotes SAPO and orange denotes GRPO.}
    \label{fig:diagnostic-video-games}
    \vspace{-3mm}
\end{figure}

\begin{figure}[h]
    \centering
    \includegraphics[width=\linewidth]{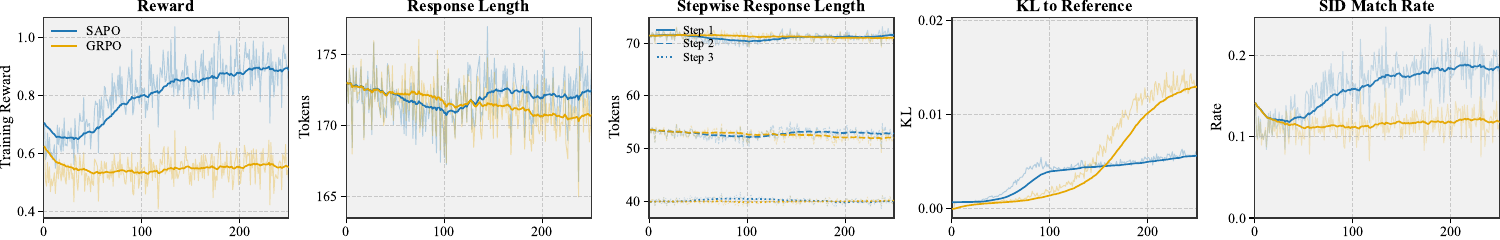}
    \caption{RL training dynamics on the \texttt{Office-Products} dataset.
    }
    \label{fig:diagnostic-office-products}
    \vspace{-3mm}
\end{figure}

\section{Extended Training Dynamics}
\label{app:diagnostic-extra}

Figures~\ref{fig:diagnostic-video-games} and~\ref{fig:diagnostic-office-products} extend the main-text diagnostic from \texttt{Industrial-and-Scientific} to the remaining two categories.
On \texttt{Video-Games}, GRPO exhibits a stronger instability pattern, with reward collapse or oscillation, response-length drift, rapid KL growth, and an unstable SID match rate, whereas SAPO keeps reward and SID-match dynamics more controlled.
On \texttt{Office-Products}, the difference is milder but follows the same direction: SAPO improves reward and SID match rate more steadily, keeps response length comparable, and shows slower KL growth than GRPO.
Together with Figure~\ref{fig:diagnostic}, these per-category diagnostics are consistent with the RQ2 conclusion that reasoning-step credit assignment stabilizes sparse exact-match RL across datasets rather than only in the representative main-text run.

% \section{Extended Ablation Results}
% \label{app:ablation-extra}
% TODO. Backbone study ($3{\times}2$ grid) and
% component study results on \texttt{Video-Games} and \texttt{Industrial-and-Scientific}.
% Same format as Table~\ref{tab:ablation} in the main text.

\section{Extended Case Studies}
\label{app:case-study}
\newcommand{\appcasesidraw}[1]{\textcolor{red!70!black}{\texttt{#1}}}

This appendix provides additional qualitative examples complementing
Figure~\ref{fig:case-study}.  Each example follows the same layout as
Figure~\ref{fig:case-study}: a shared history and target appear first, followed
by side-by-side outcome-GRPO and SAPO reasoning traces.  SID tokens are shown in
red.  The thinking blocks below are copied from the raw model outputs; malformed
outcome-GRPO outputs are shown as raw malformed thinking blocks rather than
rewritten summaries.

\noindent\textbf{Malformed Reasoning and Lost Series Signal (\texttt{Video-Games}).}
\begin{tcolorbox}[enhanced, width=\linewidth, colback=gray!4, colframe=gray!55, boxrule=0.4pt, arc=1mm, left=3pt, right=3pt, top=3pt, bottom=3pt, fontupper=\scriptsize]
\textbf{Shared prompt history.}
\appcasesidraw{<a\_250><b\_165><c\_76>} (Super Mario World);
\appcasesidraw{<a\_219><b\_191><c\_170>} (SNES-compatible console);
\appcasesidraw{<a\_141><b\_253><c\_131>} (Star Wars: Rogue Squadron);
\appcasesidraw{<a\_250><b\_14><c\_196>} (Super Mario 64);
\appcasesidraw{<a\_205><b\_82><c\_228>} (Cruis'n USA);
\appcasesidraw{<a\_175><b\_113><c\_236>} (Donkey Kong 64).\\[-1pt]
\textbf{Target.}
\appcasesidraw{<a\_250><b\_183><c\_190>}: Yoshi's Story.
\end{tcolorbox}
\vspace{2pt}
\begin{minipage}[t]{0.53\linewidth}
\begin{tcolorbox}[enhanced, width=\linewidth, height=5.3cm, valign=top, colback=red!3, colframe=red!45!black, colbacktitle=red!12, coltitle=black, title=\textbf{Outcome GRPO}, fonttitle=\scriptsize, fontupper=\tiny, boxrule=0.4pt, arc=1mm, left=3pt, right=3pt, top=3pt, bottom=3pt]
\textbf{Thinking block 1.}
16-bit Entertainment System - Playstation 2, identified as \appcasesidraw{<a\_219><b\_191><c\_170>}, is a classic video game console that revolutionized the gaming industry with its extensive library of games and multimedia capabilities. This console is particularly appealing to users who appreciate retro gaming and nostalgia for the iconic PlayStation 2 era.\\[1pt]
\textbf{Thinking block 2.}
16-bit Entertainment System - Playstation 2, identified as \appcasesidraw{<a\_219><b\_191><c\_170>}, is a classic video game console that revolutionized the gaming industry with its extensive library of games and multimedia capabilities. This console is particularly appealing to users who appreciate retro gaming and nostalgia for the iconic PlayStation 2 era.\\[1pt]
\textbf{Thinking block 3.}
28-in-1 Blue Game Card Case Holder Cartridge Box for Nintendo 3DS Video Games by NYKKOLA, identified as \appcasesidraw{<a\_219><b\_116><c\_229>}, is a practical accessory for gamers who want to keep their game cards organized and easily accessible. This product is particularly appealing to users who own multiple Nintendo 3DS games and are looking for a way to store and manage their game cartridges efficiently.\\[1pt]
\textbf{Prediction.}
\appcasesidraw{<a\_219><b\_116><c\_229>}: 28-in-1 Nintendo 3DS game-card case.
\textit{\textbf{Wrong; the target is not in top-10.}}
\end{tcolorbox}
\end{minipage}\hfill
\begin{minipage}[t]{0.46\linewidth}
\begin{tcolorbox}[enhanced, width=\linewidth, height=5.3cm, valign=top, colback=blue!3, colframe=blue!50!black, colbacktitle=blue!12, coltitle=black, title=\textbf{SAPO}, fonttitle=\scriptsize, fontupper=\tiny, boxrule=0.4pt, arc=1mm, left=3pt, right=3pt, top=3pt, bottom=3pt]
\textbf{Thinking block 1.}
Analyzing the user's interaction history, I notice a strong inclination towards classic Nintendo games, particularly those that feature platforming and adventure elements. The user has engaged with multiple titles from the Super Mario franchise, such as \appcasesidraw{<a\_250><b\_165><c\_76>} and \appcasesidraw{<a\_250><b\_14><c\_196>}, indicating a preference for nostalgic gameplay experiences that emphasize exploration and character-driven narratives.\\[1pt]
\textbf{Thinking block 2.}
Additionally, the user has shown interest in action-adventure games like \appcasesidraw{<a\_141><b\_253><c\_131>}, which suggests a broader appreciation for immersive storytelling and dynamic gameplay mechanics. The presence of \appcasesidraw{<a\_205><b\_82><c\_228>} further indicates a willingness to explore different genres, yet the predominant focus remains on Nintendo's iconic titles.\\[1pt]
\textbf{Thinking block 3.}
Given this pattern, it's reasonable to predict that the user would be drawn to another classic platforming experience that aligns with their established interests. So the predicted next item is: \appcasesidraw{<a\_250><b\_183><c\_190>}\\[1pt]
\textbf{Prediction.}
\appcasesidraw{<a\_250><b\_183><c\_190>}: Yoshi's Story.
\textit{\textbf{Correct.}}
\end{tcolorbox}
\end{minipage}

\noindent\textit{Takeaway.} Outcome-reward training produces malformed/repetitive blocks, whereas SAPO preserves the game-series signal.

\noindent\textbf{Final-SID Refinement (\texttt{Industrial-and-Scientific}).}

\begin{tcolorbox}[enhanced, width=\linewidth, colback=gray!4, colframe=gray!55, boxrule=0.4pt, arc=1mm, left=3pt, right=3pt, top=3pt, bottom=3pt, fontupper=\scriptsize]
\textbf{Shared prompt history.}
\appcasesidraw{<a\_79><b\_24><c\_40>} (thermal paste);
\appcasesidraw{<a\_36><b\_103><c\_152>} (digital audio converter);
\appcasesidraw{<a\_15><b\_10><c\_53>} (Industrial \& Scientific category);
\appcasesidraw{<a\_94><b\_242><c\_43>} (1000mL filtering flask).\\[-1pt]
\textbf{Target.}
\appcasesidraw{<a\_94><b\_93><c\_57>}: Buchner-style filter funnel.
\end{tcolorbox}
\vspace{2pt}
\begin{minipage}[t]{0.55\linewidth}
\begin{tcolorbox}[enhanced, width=\linewidth, height=5.0cm, valign=top, colback=red!3, colframe=red!45!black, colbacktitle=red!12, coltitle=black, title=\textbf{Outcome GRPO}, fonttitle=\scriptsize, fontupper=\tiny, boxrule=0.4pt, arc=1mm, left=3pt, right=3pt, top=3pt, bottom=3pt]
\textbf{Thinking block 1.}
I notice the user's initial interest in \appcasesidraw{<a\_79><b\_24><c\_40>} suggests a focus on technical components, possibly for computing or electronics, indicating a preference for precision and performance. The subsequent engagement with \appcasesidraw{<a\_36><b\_103><c\_152>} shows an openness to practical, utility-driven devices that facilitate home or workspace improvements. Moving forward, the user's selection of \appcasesidraw{<a\_15><b\_10><c\_53>} and \appcasesidraw{<a\_94><b\_242><c\_43>} points to a growing interest in laboratory equipment.\\[1pt]
\textbf{Thinking block 2.}
I also observe a tendency toward items that offer versatility and compatibility with various setups, as seen in the choice of \appcasesidraw{<a\_36><b\_103><c\_152>}. The user may enjoy exploring further laboratory accessories, specialized glassware, or other scientific tools that support hands-on experimentation.\\[1pt]
\textbf{Thinking block 3.}
Overall, the user appears motivated by products that enable precise work, facilitate experimentation, and improve their workspace or laboratory setup. I will recommend the next item as \appcasesidraw{<a\_94><b\_93><c\_171>}, which aligns with their demonstrated interest in reliable, versatile laboratory equipment.\\[1pt]
\textbf{Prediction.}
\appcasesidraw{<a\_94><b\_93><c\_171>}: generic glass funnel.
\textit{\textbf{Wrong final SID step.}}
\end{tcolorbox}
\end{minipage}\hfill
\begin{minipage}[t]{0.44\linewidth}
\begin{tcolorbox}[enhanced, width=\linewidth, height=5.0cm, valign=top, colback=blue!3, colframe=blue!50!black, colbacktitle=blue!12, coltitle=black, title=\textbf{SAPO}, fonttitle=\scriptsize, fontupper=\tiny, boxrule=0.4pt, arc=1mm, left=3pt, right=3pt, top=3pt, bottom=3pt]
\textbf{Thinking block 1.}
I observe that the user has consistently engaged with items related to laboratory and technical equipment, such as \appcasesidraw{<a\_79><b\_24><c\_40>}, \appcasesidraw{<a\_36><b\_103><c\_152>}, \appcasesidraw{<a\_15><b\_10><c\_53>}, and \appcasesidraw{<a\_94><b\_242><c\_43>}. These selections suggest a strong interest in scientific apparatus, particularly those supporting experimental setups, precise measurements, and chemical handling.\\[1pt]
\textbf{Thinking block 2.}
The user appears to value durability, compatibility with lab environments, and items that facilitate hands-on scientific work. Their choices indicate a preference for tools and accessories that enhance laboratory efficiency and reliability.\\[1pt]
\textbf{Thinking block 3.}
Given this pattern, the user may enjoy further exploring laboratory glassware, measurement instruments, or other specialized equipment that supports scientific experimentation and technical tasks.\\[1pt]
\textbf{Prediction.}
\appcasesidraw{<a\_94><b\_93><c\_57>}: Buchner-style filter funnel.
\textit{\textbf{Correct.}}
\end{tcolorbox}
\end{minipage}

\noindent\textit{Takeaway.} Both methods keep the labware prefix, but SAPO corrects the final SID step using the filtering-flask cue.

\noindent\textbf{Over-Continuity under a Sharp Product-Family Shift (\texttt{Office-Products}).}
% \noindent\textbf{\texttt{Office-Products}.}
\begin{tcolorbox}[enhanced, width=\linewidth, colback=gray!4, colframe=gray!55, boxrule=0.4pt, arc=1mm, left=3pt, right=3pt, top=3pt, bottom=3pt, fontupper=\scriptsize]
\textbf{Shared prompt history.}
\appcasesidraw{<a\_102><b\_62><c\_137>} (Sharpie yellow highlighters);
\appcasesidraw{<a\_102><b\_10><c\_219>} (Sharpie orange highlighters).\\[-1pt]
\textbf{Target.}\\[-1pt] \appcasesidraw{<a\_187><b\_147><c\_193>}\\[-1pt] (Pendaflex file folders).
\end{tcolorbox}
\vspace{2pt}
\begin{minipage}[t]{0.47\linewidth}
\begin{tcolorbox}[enhanced, width=\linewidth, height=4cm, valign=top, colback=red!3, colframe=red!45!black, colbacktitle=red!12, coltitle=black, title=\textbf{Outcome GRPO}, fonttitle=\scriptsize, fontupper=\tiny, boxrule=0.4pt, arc=1mm, left=3pt, right=3pt, top=3pt, bottom=3pt]
\textbf{Thinking block 1.}
I notice that the user has interacted with two items, both of which are high-quantity, functional office supplies designed for marking and highlighting text. Both items share the attribute \appcasesidraw{<a\_102>}, indicating a strong preference for a particular brand or product line.\\[1pt]
\textbf{Thinking block 2.}
The user's choices suggest a focus on practicality and efficiency, with an emphasis on tools that facilitate organization and document management. The repeated selection of items with attributes \appcasesidraw{<b\_62>} and \appcasesidraw{<b\_10>} points to a tendency toward products that enhance productivity and support workflow.\\[1pt]
\textbf{Thinking block 3.}
Given this pattern, the user may enjoy other office supplies that offer organizational benefits, such as labeling, sorting, or document management solutions.\\[1pt]
\textbf{Prediction.}
\appcasesidraw{<a\_187><b\_147><c\_193>}: Pendaflex file folders.
\textit{\textbf{Correct.}}
\end{tcolorbox}
\end{minipage}\hfill
\begin{minipage}[t]{0.52\linewidth}
\begin{tcolorbox}[enhanced, width=\linewidth, height=4cm, valign=top, colback=blue!3, colframe=blue!50!black, colbacktitle=blue!12, coltitle=black, title=\textbf{SAPO}, fonttitle=\scriptsize, fontupper=\tiny, boxrule=0.4pt, arc=1mm, left=3pt, right=3pt, top=3pt, bottom=3pt]
\textbf{Thinking block 1.}
I notice the user has engaged with two items, both featuring the \appcasesidraw{<a\_102>} attribute, which points to a consistent preference for a particular brand or product line. The first interaction involved \appcasesidraw{<a\_102><b\_62><c\_137>}, a highlighter with a chisel tip, while the second was \appcasesidraw{<a\_102><b\_10><c\_219>}, also a highlighter but with a different color and a larger pack size.\\[1pt]
\textbf{Thinking block 2.}
Both items emphasize practicality--bright color, smear-resistant ink, and versatility for highlighting text. The user appears motivated by utility and efficiency, possibly for tasks like studying, organizing, or office work.\\[1pt]
\textbf{Thinking block 3.}
Given this pattern, the user may enjoy other stationery items that offer similar functional benefits, such as multi-pack highlighters, markers, or organizational tools with ergonomic features or enhanced usability.\\[1pt]
\textbf{Prediction.}
\appcasesidraw{<a\_102><b\_220><c\_138>}: Sharpie fluorescent pink highlighters.
\textit{\textbf{Wrong.}}
\end{tcolorbox}
\end{minipage}

\noindent\textit{Takeaway.} SAPO does not eliminate all errors: when the next item breaks sharply from the observed family, it can overuse local continuity.

\noindent\textbf{Recovering a Recency-Sensitive Product-Family Cue (\texttt{Office-Products}).}
\begin{tcolorbox}[enhanced, width=\linewidth, colback=gray!4, colframe=gray!55, boxrule=0.4pt, arc=1mm, left=3pt, right=3pt, top=3pt, bottom=3pt, fontupper=\scriptsize]
\textbf{Shared prompt history.}
\appcasesidraw{<a\_41><b\_50><c\_26>} (Brother ink);
\appcasesidraw{<a\_41><b\_41><c\_111>} (Canon toner);
\appcasesidraw{<a\_170><b\_39><c\_252>} (Avery foam stamp pad);
\appcasesidraw{<a\_18><b\_81><c\_41>} (Avery 3-ring binder);
\appcasesidraw{<a\_130><b\_228><c\_138>} (Staples invisible tape);
\appcasesidraw{<a\_183><b\_150><c\_4>} and \appcasesidraw{<a\_183><b\_150><c\_127>} (LINKYO Canon toner);
\appcasesidraw{<a\_255><b\_117><c\_30>} (BIC black ballpoint pens);
\appcasesidraw{<a\_31><b\_53><c\_79>} (ACCO mini binder clips);
\appcasesidraw{<a\_255><b\_117><c\_124>} (BIC blue ballpoint pens).\\[-1pt]
\textbf{Target.}
\appcasesidraw{<a\_255><b\_117><c\_102>}: BIC Round Stic red ballpoint pens.
\end{tcolorbox}
\vspace{2pt}
\begin{minipage}[t]{0.49\linewidth}
\begin{tcolorbox}[enhanced, width=\linewidth, colback=red!3, colframe=red!45!black, colbacktitle=red!12, coltitle=black, title=\textbf{Outcome GRPO}, fonttitle=\scriptsize, fontupper=\tiny, boxrule=0.4pt, arc=1mm, left=3pt, right=3pt, top=3pt, bottom=3pt]
\textbf{Thinking block 1.}
Examining the user's interaction history, I notice a strong focus on essential office supplies, particularly those related to printing and document management. The repeated selection of items like \appcasesidraw{<a\_41><b\_50><c\_26>}, \appcasesidraw{<a\_41><b\_41><c\_111>}, \appcasesidraw{<a\_183><b\_150><c\_4>}, and \appcasesidraw{<a\_183><b\_150><c\_127>} suggests a preference for reliable, high-yield consumables, likely to support frequent printing tasks.\\[1pt]
\textbf{Thinking block 2.}
The inclusion of \appcasesidraw{<a\_170><b\_39><c\_252>} and \appcasesidraw{<a\_130><b\_228><c\_138>} points to an interest in versatile, everyday office tools, while \appcasesidraw{<a\_18><b\_81><c\_41>} and \appcasesidraw{<a\_31><b\_53><c\_79>} indicate a need for organization and document handling. The user also shows a tendency to purchase in bulk, as seen with \appcasesidraw{<a\_255><b\_117><c\_30>} and \appcasesidraw{<a\_255><b\_117><c\_124>}, which may reflect a preference for value and long-term supply.\\[1pt]
\textbf{Thinking block 3.}
Given these patterns, the user may enjoy exploring additional high-capacity, cost-effective printing supplies, organizational accessories, or bulk office consumables that support efficient document processing and management.\\[1pt]
\textbf{Prediction.}
\appcasesidraw{<a\_183><b\_150><c\_15>}: LINKYO HP toner cartridge.
\textit{\textbf{Wrong; the target is not in top-10.}}
\end{tcolorbox}
\end{minipage}\hfill
\begin{minipage}[t]{0.49\linewidth}
\begin{tcolorbox}[enhanced, width=\linewidth, colback=blue!3, colframe=blue!50!black, colbacktitle=blue!12, coltitle=black, title=\textbf{SAPO}, fonttitle=\scriptsize, fontupper=\tiny, boxrule=0.4pt, arc=1mm, left=3pt, right=3pt, top=3pt, bottom=3pt]
\textbf{Thinking block 1.}
Examining the user's interaction history, I notice a strong focus on essential office and productivity supplies, particularly those related to printing and document management. The repeated selection of items like \appcasesidraw{<a\_41><b\_50><c\_26>}, \appcasesidraw{<a\_41><b\_41><c\_111>}, \appcasesidraw{<a\_183><b\_150><c\_4>}, and \appcasesidraw{<a\_183><b\_150><c\_127>} suggests a preference for reliable, high-yield consumables, likely to support frequent printing needs.\\[1pt]
\textbf{Thinking block 2.}
The inclusion of \appcasesidraw{<a\_170><b\_39><c\_252>} and \appcasesidraw{<a\_130><b\_228><c\_138>} points to an interest in versatile, everyday office tools, while \appcasesidraw{<a\_18><b\_81><c\_41>} and \appcasesidraw{<a\_31><b\_53><c\_79>} indicate a practical approach to organizing and managing paperwork. The user also shows a tendency to purchase in bulk, as seen with \appcasesidraw{<a\_255><b\_117><c\_30>} and \appcasesidraw{<a\_255><b\_117><c\_124>}, which may reflect a need for cost-effective, high-volume writing instruments.\\[1pt]
\textbf{Thinking block 3.}
Overall, this user tends to prefer functional, value-oriented supplies that support efficient document handling and organization, and may enjoy exploring additional bulk packs, multi-color options, or other practical office essentials.\\[1pt]
\textbf{Prediction.}
\appcasesidraw{<a\_255><b\_117><c\_102>}: BIC Round Stic red ballpoint pens.
\textit{\textbf{Correct.}}
\end{tcolorbox}
\end{minipage}

\noindent\textit{Takeaway.} SAPO preserves a recency-sensitive product-family cue that outcome GRPO washes out under a broader office-supplies summary.